\documentclass[a4paper,11pt]{article}

\usepackage[T1]{fontenc}
\usepackage[utf8]{inputenc}
\usepackage{amsmath, amssymb, amsthm, mathtools}
\usepackage{graphicx}
\usepackage{xcolor}
\usepackage{geometry}
\usepackage{mathtools}
\usepackage{tikz}
\usepackage{enumitem}
\usepackage{minted}
\usepackage[normalem]{ulem}
\usepackage{bbm}
\usepackage{bm}
\usepackage{dsfont}
\usepackage{booktabs}
\usepackage{fullpage}

\usepackage{algorithm}
\usepackage{algpseudocode}

\usepackage{caption}

\newtheorem{assumption}{Assumption}
\newtheorem{proposition}{Proposition}
\newtheorem{corollary}{Corollary}
\newtheorem{lemma}{Lemma}
\newtheorem{theorem}{Theorem}
\newtheorem*{preuve}{Proof}

\def\bX{\mathbf{X}} 

\def\bZ{\mathbf{Z}}

\def\bV{\mathbf{V}}
\def\bS{\mathbf{S}}
\def\bH{\mathbf{H}}
\def\bW{\mathbf{W}}

\def\bb{\mathbf{b}}  

\newcommand{\R}{I\!\!R}
\newcommand{\N}{I\!\!N}

\newcommand{\Pp}{\mathbbm{P}}
\newcommand{\E}{\mathbbm{E}}

\newcommand{\1}{\mathbbm{1}}

\usepackage{array}
\usepackage{multirow}
\usepackage{makecell} 

\title{WTNN: Weibull-Tailored Neural Networks for survival analysis}
\author{Gabrielle Rives, Olivier Lopez, Nicolas Bousquet}

\author{ 
  \textsc{Gabrielle Rives} \\[2mm]
  \small SIMMT\\
  \small 11 Route des Docks, 78000 Versailles, France \\
  \small \& LPSM, Sorbonne Université \\
  \small 4 place Jussieu, 75005 Paris, France \\
  \small \texttt{gabrielle.rives@intradef.gouv.fr}
  \and \\
  \textsc{Olivier Lopez} \\[2mm]
  \small CREST Laboratory CNRS UMR 9194\\
  \small  Groupe des Écoles Nationales d'Économie et Statistique \\ \small Ecole Polytechnique, Institut Polytechnique de Paris, \\
  \small 5 avenue Henry Le Chatelier 91120 Palaiseau, France \\
  \small \texttt{olivier.lopez@ensae.fr} \\
  \and
  \textsc{Nicolas Bousquet} \\[2mm]
  \small EDF R{\&}D -- SINCLAIR Laboratory \\
  \small 7 Bd Gaspard Monge, 91120 \small Palaiseau, France \\
  \small \& LPSM, Sorbonne Université \\
  \small 4 place Jussieu, 75005 Paris, France \\
\small \texttt{nicolas.bousquet@sorbonne-universite.fr}
}

\date{}

\begin{document}

\maketitle

\begin{abstract}

The Weibull distribution is a commonly adopted choice for modeling the survival of systems subject to maintenance over time. When only proxy indicators and censored observations are available, it becomes necessary to express the distribution’s parameters as functions of time-dependent covariates. Deep neural networks provide the flexibility needed to learn complex relationships between these covariates and operational lifetime, thereby extending the capabilities of traditional regression-based models. Motivated by the analysis of a fleet of military vehicles operating in highly variable and demanding environments, as well as by the limitations observed in existing methodologies, this paper introduces WTNN, a new neural network-based modeling framework specifically designed for Weibull survival studies. The proposed architecture is specifically designed to incorporate qualitative prior knowledge regarding the most influential covariates, in a manner consistent with the shape and structure of the Weibull distribution. Through extensive numerical experiments, we show that this approach can be reliably trained on proxy and right-censored data, and is capable of producing robust and interpretable survival predictions that can improve existing approaches. \\

\end{abstract}

\section{Introduction}

In complex operational contexts, such as military logistics, complex construction projects or emergency responses, selecting the most suitable vehicles from a heterogeneous fleet is a critical task \cite{abbass2006identifying,shetty2008priority}. Missions may involve multiple tasks and operate under varying environmental and operational conditions. Ensuring mission success requires not only predictive maintenance at the subsystem level (e.g. based on estimating the Remaining Useful Life (RUL) of components \cite{Lesobre,DinhDuc}), but also a robust, system-level survival analysis that accounts for the full vehicle history and operational variability \cite{scott2022systematic,zhang2022marine}. Several approaches aim to model the predictive reliability of complex systems as vehicles by analyzing the behavior of individual subsystems. Among them, Health and Usage Monitoring Systems (HUMS; see \cite{al2019health} for a recent review in the military field) stand out for collecting real-time data on critical components. While valuable for early fault detection, HUMS often focus on isolated elements and require broader integration to support a system-level perspective \cite{de2016health,kappas2017hums,ranasinghe2022advances}. Other methods include statistical models, physics-based simulations, and digital twins which can represent the system holistically \cite{kapteyn2021predictive,zhong2023overview,wasnik2024conceptual}. However, applying these techniques to military vehicles (or other vehicles used for critical missions) is particularly challenging due to highly variable mission profiles, extreme operational conditions, and limited or classified data. These constraints make it difficult to extrapolate full-system reliability solely from subsystem behavior, highlighting the need to combine local monitoring with system-level analysis \cite{oszczypala2022reliability,dalzochio2023predictive,oszczypala2024copula}.

For this reason, it generally appears necessary to link the capacity to accomplish the mission to a notion of vehicle survival, itself linked to a statistical modeling  of the vehicle "lifespan" $T$, or an operational proxy of it--namely its temporal capacity to meet mission requirements. 
Incorporating regression processes with respect to time-dependent covariates $\bX_t$, which reflect past mission environments, downtime periods, and maintenance events, such statistical models must also account for the fact that many observations of $T$ are naturally censored. When a vehicle successfully completes a mission defined for a given duration $t_0$, it likely could have completed missions of longer durations as well. As a result, the amount of statistical information provided by system-level (e.g., vehicle-level) data is generally limited, making it necessary to restrict the dimensionality of the models. This argues in favor of representing $T|\bX_t$ 
using a 
parametric statistical model.

To mitigate the impact of this censoring, one may  define 
$T$ not as the actual service life but rather as the \emph{interval between successive downtime periods}. Each downtime necessarily involves some form of maintenance, yet it may also include a phase of immobilization (e.g., storage, relocation of the vehicle by alternative means, etc. ; see Figure \ref{fig:illustration_durations} for an illustration). In this context, the uncensored values of $T$ correspond to the durations of successful missions, whereas any mission terminated before completion, or still underway, naturally constitutes right‑censoring in the statistical distribution of $T$. \\

\begin{figure}[hbtp]
\centering
\includegraphics[scale=0.6]{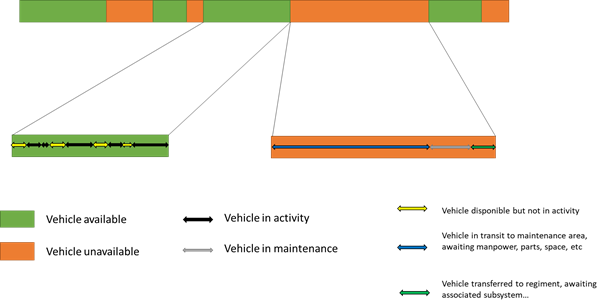}
\caption{Typical distribution of vehicle availability and unavailability durations.}
\label{fig:illustration_durations}
\end{figure}

Among the many models commonly used in survival analysis, the popular Weibull model  offers both versatility and interpretability, making it well-suited to describe different types of system-level degradation. In contrast, the Cox model is for instance limited by its proportional hazards assumption, which implies that covariate effects are time-independent, and frailty extensions are needed for its application to such lifetime data \cite{lin2002modeling}. 
Besides, for several years, the Weibull model
\begin{eqnarray}
T|\bX_t & \sim & {\cal{W}}\left(\eta(\theta,\bX_t),\beta(\theta,\bX_t)\right), \label{hyp:weibull}
\end{eqnarray}
where the scale and shape parameters $(\eta,\beta)$ are $\theta-$parametric functions of $\bX_t$, 
has received particular attention in the context of modeling covariates that characterize the operating environment of complex systems. 

By defining $T$ as the interval between successive downtime events, the standard justifications for choosing a Weibull distribution still apply. Even indirectly, $T$ embodies an underlying operational‑fatigue process. Besides, the mission durations can be assimilated to observations used for Accelerated Lifetime Testing (ALT), for which the Weibull choice is usual. Finally, the semi‑parametric structure of model (\ref{hyp:weibull}) should provide the flexibility required to generate distributions that can adapt to each vehicle over time.

One of the most popular approaches for modeling assumed Weibull-type behavior based on covariates is that of virtual age models \cite{doyen2019generic}, notably including Age Reduction After Repair (ARA) models \cite{ARA2004,liu2020unobserved,brissaud2022reliability}. While models based on non-homogeneous Poisson processes, which model failures as a Poisson process with time-varying intensity, may appear a priori to be well suited for the context considered, they remain somewhat limited. Specifically, they would estimate maintenance global effects and provide similar scale parameters for all similar vehicles of a given fleet, while an operational goal would be to discriminate between them for a given mission. 
Moreover, the covariate regression functions proposed in \cite{VirtualAge2020}, applied to the conditional power law failure intensity, are assumed to be linear, which restricts the expressiveness of the covariates $\bX_t$. Furthermore,  virtual age models require assumptions regarding the nature of the virtual age resulting from maintenance, which appear difficult to validate in cases where maintenance actions are potentially diverse and involve vehicles that may experience highly heterogeneous operating conditions.

This article therefore seeks to explore 
settings where the number of covariates may be large, going beyond ARA models with covariates \cite{VirtualAge2020}, functional Weibull models \cite{perot2017functional} 
and ALT approaches. To learn potentially complex patterns linking $\bX_t$ and $T$ and to allow for the integration of latent dependencies, recent approaches broadly reviewed by \cite{wiegrebe2024deep} 
have attempted to represent lifetime model parameters using neural networks (NN). Among them, recurrent NN architectures chosen for Weibull parameters $\{\eta(\theta,\bX_t),\beta(\theta,\bX_t)\}$ by \cite{aydin2017using} (LSTM) or \cite{WTTERNNW} require large amounts of complete (non-missing) data, and follow-up studies have reported degraded performance across various datasets. The lack of interpretability and reproducibility of this approach and its derivatives (e.g. \cite{borst2023introducing,xu2022method}) has also been criticized \cite{ndao2025improving},  particularly for their operational limitations.

NN-based approaches are generally favored in the presence of large sample sizes and have thus been widely explored in the medical domain, which typically relies on the Cox model (e.g., \cite{katzman2018deepsurv,roblin2020use}). However, these approaches do not account for the specificities of industrial maintenance, such as intervention events. Two recent adaptations of the approach proposed by \cite{katzman2018deepsurv}, namely DeepWeiSurv \cite{bennis2020estimation} and DPWTE \cite{bennis2021dpwte}, incorporate mixtures of Weibull distributions and feedforward NN. These models currently seem to stand among the most advanced candidates for survival modeling in complex systems, alongside the WTTE-RNN model \cite{WTTERNNW}. 
Preliminary studies, however, have shown that these models remain largely difficult to interpret and adapt to other contexts (data) than the one they were developed for. In particular, they do not fully addresses the challenges posed by heterogeneous vehicle usage and extended periods of unavailability. 

Hence, the goal of this work is to design regression functions for the Weibull parameters $(\eta, \beta)$ that are capable of capturing the informational richness conveyed by the covariates $\bX_t$, while simultaneously improving interpretability and maintaining the model's ability to learn from realistically sized datasets (on the order of a few thousand samples), including right-censored observations. The intended use of such functions is motivated by the ranking of similar military (or similar) vehicles in order to carry out a future mission, based on past data. In this context, we propose two main contributions.

The first one deals with the design  of a neural network called WTNN, with a shared architecture between $\eta$ and $\beta$, 
which can incorporate  prior  knowledge about the link between certain dimensions of $\bX_t$ and the survival, consistently with the geometry of the Weibull distribution. For example, the hazard rate may be a monotonic function of some covariates, a property commonly viewed as enhancing interpretability and promoting acceptance of deep learning models in industrial contexts \cite{runje2023constrained}. 
Additionally, this section presents theoretical results showing that the model can deliver estimators with strong consistency properties even in the presence of censoring. These findings offer objective guidance for architectural design, enabling the construction of variants of a standard MLP (\emph{Multi-Layer Perceptron})—for instance, by selecting the number of layers or imposing geometric constraints.

The second contribution consists in a comparative evaluation of WTNN with most relevant existing models previously discussed, using 
real-world dataset that motivated this study. We show that our approach can improve other methods in both accuracy and robustness, across various evaluation metrics related to survival prediction. 

More specifically, this paper is structured as follows. The type of datasets motivating this study is introduced in Section~\ref{sec:data-notations}, which also presents the main notations. The modeling choices and theoretical arguments behind WTNN  and the training procedure are detailed in Section~\ref{sec:modeling}. 
Section~\ref{sec:experiments} is dedicated to both simulated and real-world numerical experiments, starting with the selection of evaluation metrics and a description of the proposed simulations. Finally, the discussion section concludes this work by summarizing the main contributions and outlining future research directions. Several appendices and a Supplementary Material provide technical details aimed at enhancing both the interpretability and reproducibility of the proposed approach.

\section{Notations and data}\label{sec:data-notations}

Through a statistical description and interpretation of data from real-world applications motivating this research, we formally introduce the notations and key underlying concepts.

\paragraph{Mission durations.}\label{par:failure} 
Let us consider a fleet of $N$ vehicles with identical intrinsic characteristics. Throughout this article, a vehicle is denoted by the index $i \in \{1, \ldots, N\}$ and is assumed to have attempted $n(i)$ missions, each described by a duration and a set of observed covariates. Within the previously introduced context, an observed duration of the mission $j \in \{1, \ldots, n(i)\}$ performed by vehicle $i$ is defined as a realisation of 
\begin{eqnarray*}
Z^{(i)}_j & = & \min\left(T^{(i)}_j,Y^{(i)}_j \right)
\end{eqnarray*}
where $Y^{(i)}_j$ is  the time allotted to the mission and $T^{(i)}_j$ the mission's actual duration—i.e., the interval between two periods of vehicle unavailability (see Figure \ref{fig:illustration_durations}). Defining
\begin{eqnarray*}
\delta^{(i)}_j & = & \mathbbm{1}_{\left\{T^{(i)}_j \leq Y^{(i)}_j\right\}},
\end{eqnarray*}
the mission is deemed successful if \(\delta^{(i)}_j=1\). We assume that  $T^{(i)}_j$ can be modeled as a random variable defined on the probability space $\{\mathbb{R}^+, \mathcal{B}(\mathbb{R}^+), \mathbb{P}\}$, where $\mathcal{B}(\cdot)$ denotes the Borel $\sigma$-algebra and $\mathbb{P}$ is a probability measure absolutely continuous with respect to the Lebesgue measure. Hence the $\{T^{(i)}_j\}_j$ are non-homogeneous stochastic processes, and  if $\delta^{(i)}_j=0$ then
\(Y^{(i)}_j\) constitutes a right‑censoring time for an observation of \(T^{(i)}_j\). In this paper, for all $(i,j)$ the two processes $\{T^{(i)}_j\}_j$ and  $\{Y^{(i)}_j\}_j$ are assumed to be independent. Let us denote 
\begin{eqnarray*}
n_u  =  \sum\limits_{i=1}^N \sum\limits_{j=1}^{n(i)} \delta^{(i)}_j, & \ & 
n_c  =  \sum\limits_{i=1}^N \sum\limits_{j=1}^{n(i)} (1-\delta^{(i)}_j), \ \ \ \ \ \ n  =  n_u + n_c = \sum\limits_{i=1}^N n(i)
\end{eqnarray*}
the numbers of uncensored, censored and total observations (ie., the total number of missions across all vehicles), respectively.

Following the study context introduced earlier, $\mathbb{P}^{(i)}_j$ is chosen to be a Weibull distribution conditioned on a set of covariates $\mathbf{X}^{(i)}_j$ defined in a space $\chi \subset \mathbb{R}^d$, with $d < \infty$, summarizing the mission context. The conditioned Weibull distribution (\ref{hyp:weibull}) assumes that the parameter vector $\theta$ for two regression models $\eta(\theta,\bX_t),\beta(\theta,\bX_t)$ is common to all situations encountered for a given fleet, and belongs to a finite-dimensional open space $\Theta\in\R^q$ with $q<\infty$. In Section \ref{sec:modeling}, these models will be chosen as neural networks with joint backbone, whose architecture is broadly described by $\theta$.

\paragraph{Covariate description.}\label{par:covariates} 
The raw covariates considered in this study were primarily selected by domain experts and are mainly derived from on-board vehicle records collected over several years. These covariates allow for the monitoring of various operational metrics, including distance traveled, engine operating hours, frequency and duration of maintenance activities, availability and downtime periods (separated between maintenance durations and immobilizations), as well as areas of deployment. This dataset is enriched with temporal descriptors of the mission environment—such as climatic conditions or terrain characteristics—as well as indicators of vehicle usage beyond transportation (e.g., communication functions). Including ordinal and nominal variables (the latter being classically transformed using one-hot encoding), these raw covariates are formatted and structured as follows:
each row of the final covariate dataset $\bX_t=(X_{1,t},\ldots,X_{d,t})$ with $d<\infty$ corresponds to a phase in the history of a vehicle, along with identifying information (license plate, version). A phase consists of two parts: a period of availability (labeled A), followed by a period of unavailability (labeled B). Each period has its own covariates. An extraction of $\bX_t$  is provided on Table \ref{tab:extraction_covariates} while the overall information about the dataset is summarized in Table \ref{tab:data.description}. For our study, three differents fleets of vehicles are considered, without missing value. However, zero values are both present and significant (mostly between 10\% and 60\%) among the covariates. 
For example, a vehicle may present a zero kilometer consumption during a specific phase. These covariates can therefore be described as {zero-inflated} and {semi-continuous}. \\

\begin{table}[hbtp]
\centering
\caption{Extraction of a vehicles dataset. Covariantes related to vehicle availability are labelled (A), while covariates related to vehicle unavailability are labeled (B).}
\label{tab:extraction_covariates}
\begin{tabular}{ p{1.2 cm}  p{0.6cm}  p{2.2cm}  p{2.2cm}  p{2.2cm}  p{2.2cm}  p{2.2cm}}
 \hline
 Vehicle & Age & Availability  duration (A)  & Km consumption (A) & Utilization Area (A) & Immobilization duration (B) & Maintenances duration (B)  \\ 
 \hline
 1 & 8 & 2 & 25 & A & 0 & 20 \\ 
 1 & 17 & 7 & 0 & A & 2 & 200 \\
 & \\ 
 2 & 6 & 5 & 155 & B & 4 & 53 \\
 2 & 8 & 0 & 12 & C & 2 & 250 \\
 2 & 15 & 5 & 47 & D & 2 & 27 \\ [1ex] 
& \\ 
  3 & 25 & 8 & 10 & A & 3 & 5 \\ 
 3 & 26 & 1 & 0 & A & 0 & 125 \\
 3 & 29 & 2 & 5 & B & 1 & 53 \\
 3 & 31 & 1 & 250 & C & 1 & 31 \\
 3 & 34 & 2 & 123 & D & 1 & 27 \\ [1ex] 
 \hline
\end{tabular}
\end{table}

\begin{table}[hbtp]
\centering
\caption{Data description for three examples of real military vehicle fleets. Up (\emph{resp.} down) arrows indicate the covariates for which survival increases or decreases, all other things being equal (See Section \ref{sec:consistency_survival}).}
\label{tab:data.description}
\begin{tabular}{lll}
\hline
{\bf Symbol} & {\bf Meaning} & {\bf Value(s)} \\
\hline
$N$ & Fleet size & $\left\{5,000; 600; 400\right\}$ \\
$n$ & Total number of missions across all vehicles & $\left\{{95,000} \ ; \ 8,900 \  ; \ 1,700 \right\}$ \\
$d$ & Input dimension (total number of covariates) & 37 \\
$d_n$ & Number of numerical covariates & 29 \\
$d_c$ & Number of categorical covariates & 8 \\
& with numbers of categories per covariate & 5-10 \\
$\delta$ & Right-censoring indicator & \\
& with overall censoring rate & $\sim$ 5\%  \\
& \\
$\bX_t$ & Covariates collected through time, including: & \\
& \\
& 
{\it (related to vehicle availability) } & \\
& utilization area & Cat.\\
& kilometers and engine hours & Num. $\downarrow$\\
& age & Num \\
&  availability duration & Num  \\
& \\
& {\it (related to vehicle unavailability) } & \\
&  maintenance duration & Num. \\
&  number of part remplacements & Num.\\
& unavailability duration & Num  $\uparrow$ \\
& \\
\hline
\end{tabular}
\end{table}

\clearpage
\section{Modeling and training}\label{sec:modeling}

\subsection{Likelihood-based loss function}

Building on Equation~(\ref{hyp:weibull}), we assume that a parametric model governed by $\theta$ can be constructed to relate the Weibull parameters to the covariates of the problem. Such a regression model is introduced in Section~\ref{sec:nn-weib}. 

Given a dataset representing a fleet of identical vehicles, the statistical interpretation of the data proposed in the previous section implies that the probability density function of the time between unavailabilities for vehicle $i \in \{1,\ldots,N\}$ during mission $j \in \{1,\ldots,n(i)\}$ can be written, as a function of $\theta$, as follows:
\begin{eqnarray*}
\\
f^{(i)}_j(\theta|{\cal{F}}^{(i)}_{j-1})  =  \left[\frac{\beta\left(\theta,\bX^{(i)}_j\right)}{\eta\left(\theta,\bX^{(i)}_j\right)}\left(\frac{z^{(i)}_j}{\eta\left(\theta,\bX^{(i)}_j\right)}\right)^{\beta(\theta,\bX^{(i)}_j)-1}\right]^{\delta^{(i)}_{j}} 
\exp\left(-\left\{\frac{z^{(i)}_j}{\eta\left(\theta,\bX^{(i)}_j\right)}\right\}^{\beta\left(\theta,\bX^{(i)}_j\right)}\right). \\
 & & 
\end{eqnarray*}
Here, $\mathcal{F}^{(i)}_{j-1}$ denotes the $\sigma$-algebra generated by the trajectory of the observed process $\{z^{(i)}_t, \bX^{(i)}_t\}_{t \leq j-1}$, i.e., it encodes the history of vehicle $i$ up to mission $j$. Let us assume that each vehicle, after a maintenance, is considered As Good As New (AGAN = perfect maintenance ; \cite{de2017imperfect}). Then,  for all $i$, the realization of $z^{(i)}_t$ conditional on $\mathcal{F}^{(i)}_{t-1}$ is assumed to be independent on $z^{(i)}_{t-1}$. Therefore the likelihood function associated with the observed history of vehicle $i$ can be expressed as
\begin{eqnarray*}
\ell^{(i)}(\theta) & = & \prod\limits_{j =1}^{n(i)}  f^{(i)}_j(\theta \mid \mathcal{F}^{(i)}_{j-1}),
\end{eqnarray*}
where we set $\mathcal{F}^{(i)}_0 = \emptyset$. Finally, assuming that the histories of vehicles $i$ and $j$ within the same fleet are independent for all pairs $(i,j)$, the full likelihood of the $n$ observations corresponding to the fleet of $N$ vehicles becomes
\begin{eqnarray}
\ell_n(\theta) & = & \prod\limits_{i=1}^N \ell^{(i)}(\theta). \label{eq:complete_likelihood}
\end{eqnarray}
In a classical statistical learning framework, we can then derive the likelihood-based (or equivalently, the cross-entropy) loss function $L_n(\theta) = -\log \ell_n(\theta)$ to be minimized  and estimate the maximum likelihood (ML) parameter by
\begin{eqnarray}
\hat{\theta}_n & \in & \arg\min\limits_{\theta \in \Theta} L_n(\theta). \label{eq:learning1}
\end{eqnarray}
While the maximum likelihood estimation of classical Weibull models is fully understood \cite{rinne2008weibull}, 
replacing \((\eta,\beta)\) by continuous, complex regression functions in \(\theta\) 
is 
likely to produce a cost function \(L_n(\theta)\) that is non-convex neither quasi-convex (resp.\ a log-likelihood function \(\ell_n(\theta)\) that is not (log-)concave / quasi-concave on \(\Theta\)) \cite{choromanska2015loss,kawaguchi2016deep}. It is thus likely that  \(\hat\theta_n\) be not unique, except if the architecture is designed to avoid such pitfalls. For instance, choosing ReLU neural networks to model \(\beta(\theta,\bX)\) and \(\eta(\theta,\bX)\) yields piecewise convexity in \(\theta\) for these networks \cite{rister2017piecewise}, which can be enhanced via weight decay approaches \cite{milne2019piecewise} or extended by employing exponential weights together with \(L^1\) regularization \cite{bengio2005convex,zhou2022parameter}, yet without guaranteeing full convexity of \(L_n(\theta)\). In this context, it is therefore appropriate to seek architectural choices that regularize the loss landscape. 
A first approach is to ensure coherence between the architectures of \(\beta(\theta,\bX)\) and \(\eta(\theta,\bX)\) and the geometric properties of the Weibull distribution, in terms of the relationship between these parameters and survival behavior. Such an approach is proposed in \S\ref{sec:nn-weib}, on a basis of a classical MLP architecture. A study of the theoretical properties of  Weibull-tailored network is summarized in the same section, which leads to improve this base architecture.  The pursuit of regularizations \(L(\theta)\) that notably prevent overfitting, by replacing (\ref{eq:learning1}) with
\begin{eqnarray}
\hat{\theta}_n & \in & \arg\min_{\theta \in \Theta} \, \left\{L_n(\theta) + L(\theta)\right\},
\label{eq:learning2}
\end{eqnarray}
is studied in \S\ref{sec:learning}.

\subsection{Weibull-tailored neural networks }\label{sec:nn-weib}

The architecture of  \(\beta(\theta,\bX)\) and \(\eta(\theta,\bX)\) must satisfy several requirements: correlation (both parameters must evolve based on the same covariates), reasonable dimensionality (due to the limited amount of industrial data), positivity for \(\eta(\theta,\bX)\) and bounded positive outputs for \(\beta(\theta,\bX)\), and consistency with the qualitative characteristics of the Weibull survival behavior. We propose below an approach that aims to reconcile all these aspects.

\subsubsection{Isotonic consistency with Weibull survival}\label{sec:consistency_survival}

To the light of information provided in Section \ref{sec:data-notations}, decompose $\bX=(\bX^{o},\bX^{nom})$ where $\bX^{o}$  (\textit{resp.} $\bX^{nom}$) is the subset of ordinal (\textit{resp.} nominal) input variables. It means that there exists a partial ordering $\succeq$ between the elements of $\bX^{o}$:
\begin{eqnarray*}
\forall (\bX^{(i)},\bX^{(j)})\in \bX^{o}, \ \ \ \bX^{(i)} \succeq \bX^{(j)} & \Leftrightarrow & x^{(i)}_k \geq  x^{(j)}_k, \ \ \ \forall k\in\{1,\ldots,d\}
\end{eqnarray*}
Consider furthermore the following assumption about the decreasing monotonicity of vehicle survival at mission $j$
\begin{eqnarray}
p(t|\bX) := \Pp\left(T_{j}>t | \bX_j,\theta\right) & = & \exp\left(-\left\{\frac{t}{\eta\left(\theta,\bX_j\right)}\right\}^{\beta\left(\theta,\bX_j\right)}\right).\label{eq:survival}
\end{eqnarray}

\begin{assumption}\label{assumption:risk_monotonicity}
 Possibly at the price of a reparametrization, there exists a subset of covariates $X^{o_a}\subseteq X^o$ such that the survival (\ref{eq:survival}) is a decreasing function of $X^{o_a}$, all other covariates  being unchanged. 
 \end{assumption}

\noindent This assumption can be translated as follows:  $\forall \bX_i=(\bX^{o_a}_i,\bX^{o_b},\bX^{nom})$,   then 
\begin{eqnarray}
\bX^{o_a}_2 \succeq \bX^{o_a}_1 & \Rightarrow & p(t|\bX_1) \geq p(t|\bX_2) \ \ \ \forall t\geq 0. \label{hypo:monotonicity_1}
\end{eqnarray}
To be more succinct, we can rewrite Assumption  \ref{assumption:risk_monotonicity} 
as follows: 
\begin{eqnarray}
\forall t\geq 0, \ \ \ & & p(t|\bX)\downarrow \bX^{o_a}. \label{hypo:monotonicity_2}
\end{eqnarray}
For instance, it seems reasonable to assume that, all other things being equal, poorer-quality equipment, a larger number of previous repairs  or harsher environmental conditions will lead to failure more quickly. Next proposition provides a useful condition to translate this isotonic assumption to a constraint on the joint architecture of $\beta(\theta,\bX)$ and \(\eta(\theta,\bX)\). 

\begin{proposition}\label{prop:constraints}
Denote $x\mapsto \Gamma(x)$ the gamma function defined on $\mathbb{R}^+_*$ and $\xi=\arg\min_{\mathbb{R}^+} \Gamma(x)$ ($\xi\simeq 3/2$). Denote $\beta_0=(\xi-1)^{-1}$. 
Under Assumption \ref{assumption:risk_monotonicity}, then, $\forall \theta\in\Theta$,
\begin{eqnarray}
\bX\mapsto\beta(\theta,\bX)\uparrow \bX^{o_a}, & & \\ 
\bX\mapsto\eta(\theta,\bX) \downarrow \bX^{o_a} & & \ \ \text{conditionally to } \ \ \beta(\theta,\bX)>\beta_0.
\end{eqnarray}
\end{proposition}

\noindent Hence,  a two-heads neural network backbone  with $L$
layers for  $(\eta(\theta,\bX),\beta(\theta,\bX))$, 
 such that the resulting predictive Weibull model be in accordance with Assumption \ref{assumption:risk_monotonicity}, could be designed as precised beneath, using a constraint of positivity on weights (as for monotonic neural networks)  and a  branching mechanism associated to a masking activation. This positivity constraint can be implemented in various ways (e.g., \cite{runje2023constrained} or see the unconstrained alternative \cite{wehenkel2019unconstrained}), but probably the simplest approach is either defining these weights as exponential functions, or adding a specific loss \cite{gupta2019incorporatemonotonicitydeepnetworks}. 
 Besides, two additional constraint  apply on $(\eta(\theta,\bX),\beta(\theta,\bX))$:
\begin{eqnarray}
\eta(\theta,\bX) & \geq & \eta_{\min} \ > \ 0, \label{constraint-eta-1} \\
\beta(\theta,\bX) & \in & [\beta_{\min},\beta_{\max}] \label{constraint-beta-1}
\end{eqnarray}
where, typically, $\beta_{\min}=1$ to assume aging and $\beta_{\max}<6$ to avoid unplausible, over-accelerated aging \cite{lawless2011statistical}. While the strictly positive lower bound $\eta_{\min}$ will appear necessary to get consistency properties, this constraint appear mild in practice. Indeed, as $\Pp(T<\eta(\theta,\bX)|\bX)=1-\exp(-1)\simeq 0.63$, by invariance of scale transformation \cite{dasgupta2014characterization} it is very unlikely in practice that $\eta(\hat{\theta}_n,\bX)<1$ while assuming  all observed values $z^{(i)}_j\geq 1$, provided the model is not irrelevant. (The fleets motivating the study are made up of vehicles that have all been driven a minimum amount of time between two downtime periods.) Therefore it is recommended to normalize the duration observations such a way and set by defect $\eta_{\min}=1$.

\paragraph{Base architecture description.} We parameterize a base feedforward neural network (MLP) of $L$ hidden layers with a set of weight matrices and bias vectors  $(\mathbf{W}_{\ell},b_{\ell})$, where 
${\ell}\in\{1,\ldots,L\}$, such that $\mathbf{W}_{\ell}\in\mathbb{R}^{m_l \times q_{\ell}}$ and $b_{\ell}\in \mathbb{R}^{m_{\ell}}$, where $m_{\ell}$ is the number of neurons in the ${\ell}$th hidden layer and $q_{\ell}$ is the input dimension of the signal for the weighting operation. In the following, we denote $q$ the dimension of $X^{o_a}$ and $q'=d-q$ the dimension of $\{X^{o_b},X^{nom}\}$. Specific notations for weights and bias related to the branching and masking mechanisms are precised further. Assuming to set activation functions and values for the $m_{\ell}$, then $\theta$ is the set of all unknown architecture parameters:
\begin{eqnarray*}
\theta & = & \left\{(\mathbf{W}_{\ell},b_{\ell})_{\ell\in\{1,\ldots,L\}}\right\}.
\end{eqnarray*}
Starting from a baseline multilayer perceptron (MLP) formalization, consider the following specification. The resulting architecture is summarized on Figure \ref{fig:joint_architecture}. 
\begin{enumerate}
\item {\bf First layer.} Given $m_1$ neurons, apply the following transformation
\[
\mathbf{H}_1 = \phi(\mathbf{W}_1^{o_a} \bX^{o_a} + \mathbf{W}_1^{o_b} \bX^{o_b} + \mathbf{W}_1^n \bX^n + b_1),
\]

where \( \mathbf{W}_1^{o_a} \in \mathbb{R}^{m_1 \times q} \) and $\mathbf{W}_1^{o_a}\geq 0$,  \( \mathbf{W}_1^{o_b} \in \mathbb{R}^{m_1 \times q'} \), \( \mathbf{W}_1^n \in \mathbb{R}^{m_1 \times q'} \), and \( b_1 \in \mathbb{R}^{m_1} \). The {activation function} \( \phi \) must ensure monotonicity for the ordinal input (e.g., \texttt{ReLU} or \texttt{softplus}).

\item {\bf Hidden layers.} Subsequent layers combine information through standard feedforward connections, represented abstractly as:
\begin{equation}
\mathbf{H}_{{\ell}} = \sigma(\mathbf{W}_{{\ell}} \mathbf{H}_{{\ell}-1} + b_{{\ell}}), \quad {\ell} = 2, \dots, L-2,
\end{equation}
where  \( \mathbf{W}_{\ell} \in \mathbb{R}^{m_{\ell} \times m_{{\ell}-1}} \) and $\mathbf{W}_{\ell}\geq 0$, \( b_{\ell} \in \mathbb{R}^{m_{\ell}} \), and \( \sigma \) is a monotonic activation function. 
\item {\bf Branching at \( H_{L-1} \).} The final hidden layer \( \mathbf{H}_{L-1} \) splits into two branches:
\begin{itemize}
    \item {\it Shape paramet\it er \( \beta(\theta,\bX) \)}: define
    \begin{eqnarray*}
    \tilde{\beta}(\theta,\bX) & = &  f_{\beta}(\mathbf{W}_{\beta} \mathbf{H}_{L-1} + b_{\beta}), \quad \mathbf{W}_{\beta} \in \mathbb{R}^{1 \times m_{L-1}}, \quad \mathbf{W}_{\beta}\geq 0 \quad b_{\beta} \in \mathbb{R}, 
    \end{eqnarray*} 
    where \( f_{\beta} \) ensures non-negativity (e.g., \texttt{softplus}).  Then, monotonic pooling and final activation are applied to respect Constraint (\ref{constraint-beta-1}):
    $$
    \beta(\theta,\bX) = \beta_{\min} + (\beta_{\max} - \beta_{\min})  \mbox{pooling}(\tilde{\beta}(\theta,\bX)).
    $$
    
    \item {\it Scale parameter \( \eta(\theta,\bX) \)}: define
    \begin{eqnarray}
    \tilde{\eta}(\theta,\bX) & = &  \eta_{min} + \sigma_{\eta}(\mathbf{W}_{\eta} H_{L-1} + b_{\eta}), \quad \mathbf{W}_{\eta} \in \mathbb{R}^{1 \times m_{L-1}}, \quad b_{\eta} \in \mathbb{R},
    \label{eq:eta_predictor}
    \end{eqnarray}
    where \( \sigma_{\eta} \) ensures non-negativity (e.g., \texttt{softplus}), such that Constraint (\ref{constraint-eta-1}) be respected.
    Pooling is applied to \( \tilde{\eta}(\theta,\bX) \) to produce the final scale parameter:
    $$
    {\eta}(\theta,\bX) = \mbox{pooling}(\tilde{\eta}(\theta,\bX)).
    $$
    Here, \( \mathbf{W}_{\eta} \) is conditionally constrained as follows: denoting $\mathbf{V}$ a set of free weights and 
    $$
    \mbox{mask}(\bX) = {\mathds{1}}_{\{ \beta(\theta,\bX) > \beta_0\}},
    $$
    then
    $$
    \mathbf{W}_{\eta} = \mbox{mask}(\bX) (-|\mathbf{V}|) + (1 - \mbox{mask}(\bX))  V.
    $$
\end{itemize}
\end{enumerate}

\begin{figure}[hbtp]
    \centering
    \begin{tikzpicture}[node distance=1.5cm]
        \node[draw, rectangle] (Input) {Input \( \mathbf{X} = (\mathbf{X}^{o_a},\mathbf{X}^{o_b}, \mathbf{X}^{nom}) \)};
        
        \node[draw, rectangle, below of=Input, yshift=-1.5cm] (Shared) {Shared Network \( \mathbf{H}_1 \), \( ..., \), \( \mathbf{H}_{L-1} \)};
        \node[below of=Shared, node distance=0.8cm] (SharedInfo) {\( L \) layers, \( m_1, ..., m_{L-1} \) neurons};
        
        \node[draw, circle, below of=Shared, yshift=-2cm] (BetaWeights) {\( \mathbf{W}_{\beta}, b_{\beta} \)};
        \node[draw, circle, below of=BetaWeights, yshift=-1.5cm] (BetaPooling) {Pooling};
        \node[draw, rectangle, below of=BetaPooling, yshift=-1.5cm] (BetaActivation) {Final Activation: \( \beta(\theta,\mathbf{x}) \)};
        
        \node[draw, circle, right of=BetaWeights, xshift=4cm] (EtaWeights) {\( \mathbf{W}_{\eta}, b_{\eta} \)};
        \node[draw, circle, below of=EtaWeights, yshift=-1.5cm] (EtaPooling) {Pooling};
        \node[draw, rectangle, below of=EtaPooling, yshift=-1.5cm] (EtaActivation) {Final Activation: \( \eta(\theta,\mathbf{x}) \)};
        
        \draw[->] (Input) -- (Shared);
        \draw[->] (Shared) -- (BetaWeights);
        \draw[->] (BetaWeights) -- (BetaPooling);
        \draw[->] (BetaPooling) -- (BetaActivation);
        \draw[->] (Shared) -- (EtaWeights);
        \draw[->] (EtaWeights) -- (EtaPooling);
        \draw[->] (EtaPooling) -- (EtaActivation);
        
        \draw[->, dashed] (BetaActivation) -- (EtaWeights) node[midway, above, sloped] {\( \beta(\theta,\mathbf{x}) \)};
    \end{tikzpicture}
    \caption{Neural network architecture.}
    \label{fig:joint_architecture}
\end{figure}
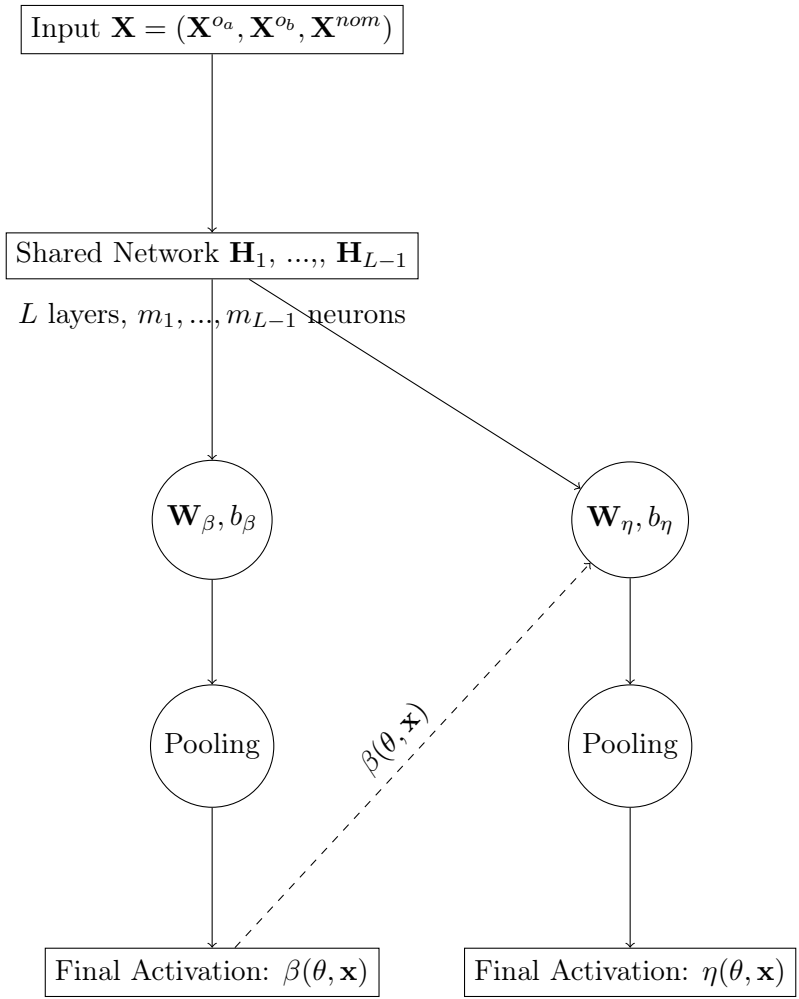

\subsubsection{From idealized theoretical properties to architectural choices}

The foregoing considerations do not shed light on the architectural choices that would favor a good estimation of $\theta$.  
Examining the possible consistency, unicity and asymptotic normality of the maximum likelihood estimator (MLE) $\hat{\theta}_n$
defined in~(\ref{eq:learning1})
offers guidance in this respect. Imposing to define a "true" parameter $\theta_0$ (which is speculative as neural networks are themselves approximations of a unknown link function), these two properties rely on the extension to functional parameters
of Chanda’s theorem~\cite{chanda1954note} by van der Vaart~\cite{van2000asymptotic,van1996weak}, which are themselves based on the following set of assumptions.  

\begin{assumption}[Compacity of the parameter space]\label{assumption:compacity_param}
 Denote $\|\cdot\|_F$ the Frobenius norm and $\|\cdot\|$ the usual Euclidian norm. There exists $R>1$ such that the set $\Theta:=\{\theta:  \ \|\bW_{\ell}\|_F\le R, \ \|b_{\ell}||\le R, \ \ell=1,\ldots,L\}$ is compact and contains $\theta_0$.   
\end{assumption}

\begin{assumption}[Compacity of the input space]\label{assumption:compacity_input}
There exists $\Delta>0$ such that  $\|\bX\|_{\infty}\leq \Delta$.    
\end{assumption}

\begin{proposition}[Uniform upper bound for $\eta$]\label{prop:upper_bound_eta}
Assume \texttt{softplus} activations $\sigma$ and denote $m_{\sup}$ the maximal number of neurons per hidden layer. Under Assumptions (\ref{assumption:compacity_param}) and (\ref{assumption:compacity_input}), there exists a finite upper bound $\eta_{max}$ on $\theta\mapsto \eta(\theta,\bX)$ over $\Theta$ that depends only on $(L,m_{\sup},R,\bX,d)$.
\end{proposition}


\begin{assumption}[Regular activation functions]\label{assumption:regularity_activations}
The activation functions are $C^2$ and their first and second order derivatives are uniformly bounded.   
\end{assumption}

\noindent The two first assumptions, associated to a choice of $C^2$ activation functions  with uniformly bounded derivatives of  (ie., \texttt{softplus}, smoothed \texttt{ReLu} (SReLu \cite{bresler2020corrective})), lead to the following proposition. 

\begin{proposition}[$C^2$ differentiability]\label{prop:C2-diff}
Under Assumptions \ref{assumption:compacity_param} to \ref{assumption:regularity_activations}, 
the mappings $\theta \mapsto \beta(\theta,\mathbf X)$
and 
$\theta \mapsto \eta(\theta,\mathbf X)$
are twice continuously differentiable; furthermore, their first- and second-order derivatives are uniformly bounded for all $\theta\in\Theta_n$ by a function that is integrable with respect to~$\mathbf X$.
\end{proposition}

\noindent Finally, to get consistency of the MLE we require: (a) that the weight of censored observations in the loss function $L_n(\theta)$ be linearly controlled, ensuring the existence of asymptotic concentration of statistical information (Assumption \ref{assum:censoring_vanish}); (b) an assumption on how the architecture of the neural network can depend on the total sample size~$n$ of observations.
More precisely, 
in this theoretical part of the study we 
temporarily adopt a \emph{sieve} approach \cite{shen2023asymptotic}, according to which the number of layers and the number of neurons per layer may depend on $n$, so as this network can approximate the potentially infinite-dimensional functional parameters $(\eta,\beta)$ (Assumption \ref{assump:gc}). 

In this sieve setting, the parameter space $\Theta$ must be replaced  by a sequence of increasing spaces for the inclusion $\Theta_{1}\subset \Theta_{2}\subset \ldots \Theta_{n}\subset \ldots$ and their union space $\Theta_{\infty}=\cup_{n=1}^{\infty} \Theta_n$ \cite{grenander1981abstract}. It is then usually assumed that $\theta_0\in \Theta_{\infty}$ (ie., the "true" parameter of infinite dimension) and that for $n$ large enough, there exists a good approximation of  $\theta_0$ in $\Theta_n$.  The MLE consistency results provided further are relative to this approximation. Else, if the true model does not belong to  $\Theta_{\infty}$, the MLE consistency is defined with respect to a pseudo-parameter $\theta^*_0$ defined as the Kullback-Leibler minimizer between the true model and its approximation in $\Theta_{\infty}$. See \cite{shen1994convergence,van1996weak} for technical details. 

Next assumption about censoring is formulated in a generalized way, requiring only that the number of censored observations grows at most linearly with the sample size. 


\begin{assumption}\label{assum:censoring_vanish}
The number of censored observations grows at most linearly with $n$:
$n_c  =  {O}\left(n\right)$. 
\end{assumption}

 In the case of i.i.d. random censoring, Assumption \ref{assum:censoring_vanish} holds with probability tending to 1, since the number of censored observations is, in average, $n\times \mathbb{P}(\delta=0).$ However, Assumption \ref{assum:censoring_vanish} allows to also cover deterministic forms of censoring.

\begin{assumption}[Sieve rate condition]\label{assump:gc}
Denote $p_n=\mbox{Card}(\theta)$. Then
$p_n=o(n^{1/3})$ and $L=o(n^{2/3})$. 
\end{assumption}




\begin{corollary}[Existence and consistency of the MLE]\label{cor:existence_consistency}
Under Assumptions \ref{assumption:compacity_param} to \ref{assump:gc}, there exists at least one estimator $\hat{\theta}_n \in\arg\max \ell_n(\theta)$. Besides, if $\theta_0\in\Theta_n$, then any maximizer $\hat{\theta}_n$ of $\ell_n(\theta)$ satisfies
$\hat{\theta}_n  \xrightarrow{\Pp}{}  \theta_0$.
\end{corollary}

\noindent To go further and obtain a controllable asymptotic behavior of the MLE, it is necessary to ensure the existence and non-degeneracy of the Fisher information. This implies two key conditions:  
(a) the two network heads, which are expected to provide distinct predictions for $\beta(\theta,\bX)$ and $\eta(\theta,\bX)$, must rely on clearly different (i.e., non-collinear) weights \((W_\beta,\,W_\eta)\in\mathbb{R}^{1 \times m_{L-1}}\) ;  
(b) the covariates must span a sufficiently rich set of scenarios so that all the network parameters are effectively involved in the prediction process (every
direction of the shared representation is activated).  
These conditions are formalized in the following (joint) assumption.

\begin{assumption}[Head identifiability and input excitation]\label{assump:distinct_heads}
Let \(h_{L-1}(X)\in\mathbb{R}^{m_{L-1}}\) be the output (activation) after the $L-1$ layer (the last shared representation in the joint backbone). For \(u,v\in\mathbb{R}^{m_{L-1}}\) we denote the standard Euclidean inner product by
\(\langle u,v\rangle := u^{\top}v = \sum_{k=1}^{m_{L-1}} u_k v_k\), and we denote strict positive definiteness using $\succ$. Remind that $(W_{\beta},W_{\eta})$ are $m_{L-1}-$dimensional vectors and 
assume 
\begin{eqnarray}
\langle W_\beta,\,W_\eta\rangle & < & \|W_\beta\|_{2}\,\|W_\eta\|_{2}, \ \ \ \text{(non-collinearity of the two heads)}
\label{eq:A7-1}\\
\mbox{Cov}\!\bigl[h_{L-1}(\bX)\bigr] &\succ & 0 \hspace{2.5cm} \ \ \text{(non-degenerate input support)} 
\label{eq:A7-2}
\end{eqnarray}    
\end{assumption}

\begin{proposition}[Regular Fisher information]
\label{prop:fisher}
Under Assumptions \ref{assumption:compacity_param} to \ref{assump:distinct_heads}, the  Fisher information matrix $\E\left[\nabla_{\theta}\log \ell_n(\theta)\nabla_{\theta}\log \ell_n(\theta)\right]$ exists and is strictly positive definite.
\end{proposition}


\noindent Satisfaction of Assumption \ref{assump:distinct_heads} can be ensured through architectural constraints or additional loss-based constraints. 
For each $\nu \in \{\eta, \beta\}$, replace each scalar (unnormalized) head $\nu(\theta,\bX) = W^T_{\nu} h_{L-1} + b_{\nu}$ with a sequence of operations (corresponding to two independent parallel mini-MLPs):
\begin{eqnarray*}
\left\{
\begin{array}{lll}
z_{\nu} & = & V_{\nu}h_{L-1} + c_{\nu} \ \ \ \ \text{($V_{\nu}\in\R^{r \times m}$)} \\
g_{\nu} & = & \texttt{softplus}(z_{\nu}), \\
\nu(\theta,\bX) & = & u_{\mu}^T g_{\nu} + d_{\nu} \ \ \ \ \text{($u_{\nu}\in\R^{r}$)}
\end{array}\right.
\end{eqnarray*}
then if $\nu=\beta$, normalize it by $\beta(\theta,\bX)=\beta_{\min}+(\beta_{\max}-\beta_{\min})\sigma(\nu(\theta,\bX))$ where $\sigma$ is a sigmoid activation. Else if $\nu=\eta$, define $\eta(\theta,\bX)=\eta_{min}+\sigma(\nu(\theta,\bX))$. 
Except in pathological confi\-gurations, the matrices $V_{\eta}, V_{\beta}$ and the vectors $u_{\eta}, u_{\beta}$ are independent, which ensures that the two gradients remain in distinct subspaces. Moreover, the \texttt{softplus} activations are strictly increasing and never null, preventing neuron deactivation. An alternative approach (which avoids introducing too many additional parameters, but is ideally used in combination) is to inject two loss terms $L_{\text{orth}}$ and $L_{\text{cov}}$, described in $\S$\ref{subsection:orth_cov}. \\


\begin{theorem}[CLT for the MLE]\label{theo:clt}
Assume the existence and identifiability of $\theta_0\in\Theta_{\infty}$.  Denote $I(\theta_0)$ the unit Fisher information matrix in $\theta_0$. Then, under Assumptions \ref{assumption:compacity_param} to \ref{assump:distinct_heads}, 
\[
\sqrt{n}(\hat\theta_n - \theta_0) \xrightarrow{d} \mathcal{N}(0, I(\theta_0)^{-1}). 
\]
\end{theorem}

\noindent These results 
provide the minimum statistical
framework expected for the modeling problem and entail local uniqueness of the MLE with probability tending to~1.  They do not, however, guarantee that a global maximiser is unique 
or that the true parameter $\theta_0$ actually exists and is identifiable. 
Nevertheless, they can offer an objective foundation for guiding the architectural design, at least as an initialization of an empirical or autoML-type search \cite{salehin2024automl}.

A structuring choice concerns the distribution of the number of neurons across layers.  
It is customary to select hidden layers of equal width, both for reasons of readability and in order to benefit from the theoretical universal approximation property of a sufficiently wide and deep MLP with constant width.  
In our study, however, such a choice would risk exacerbating the probable non-convexity of the loss function, limit identifiability by enforcing a high-dimensional representation, and would not allow verification of the previously stated sieve condition.  
We therefore suggest adopting a decreasing pyramidal architecture, with a restricted number $L$ of hidden layers. Additional constraints are to have $m_1\geq d$, to avoid successive deep layers with a same small number of neurons (no plateaus) and to avoid a last hidden layer with just one neuron, to avoid the non-collinearity of the two heads (Assumption \ref{assump:distinct_heads}).  The architectural construction can be formalized as follows. It balances a {polylogarithmic depth} $L$ with a {sub-polynomial width} $m_{\max}$ (of order $n^{1/6}/\sqrt{\log n}$), keeping the model expressive while ensuring the asymptotic control required by estimation theory.

\begin{proposition}[Architecture proposal]\label{prop:archi}
Set $n_{\min}\in\N^*$ and define $K_0=2d^2(\log n_{min})n^{-1/3}_{\min}$. Consider 
some fixed $K\geq K_0$ and $0<\rho<1$.
Consider a pyramidal width profile with geometric decay defined by
$$
\left\{
\begin{array}{lll}
m_1 &=& m_{\max} \;=\; \left\lceil \sqrt{\frac{K}{2}}\;\frac{n^{1/6}}{\sqrt{\log n}}\right\rceil, \\
m_\ell &= & \max \ \!\Bigl\{2,\ \min\bigl(\,\lceil m_1\,\rho^{\,\ell-1}\rceil,\ m_{\ell-1}-1\,\bigr)\Bigr\},\quad \ell\ge 2.
\end{array}\right.$$
For some fixed $0<\tau<1$, define the neural network depth
\[
L \;=\; \min\Bigl\{\,\big\lceil(\log n)^{\tau}\big\rceil,\ \ 1+(m_{\max}-2)\Bigr\}.
\]
Then, for $n>n_{\min}$ and $m_1>d$,   $(m_\ell)_{\ell=1}^{L_n}$ is strictly decreasing with $m_\ell\ge 2$ and Assumption~6 (sieve rate) holds.
\end{proposition}

Tables \ref{tab:structural_choices-1} and \ref{tab:structural_choices-2} report representative architectural choices derived from the preceding proposition, for the cases $d=10$ and $d=37$ covariates, respectively. The ratio $n/p_n$ serves as a heuristic measure of the sample size needed to support the estimation of an individual component of $\theta$.

To preserve expressiveness in the first layers of the network, we suggest to use $\rho = 1/2$. 
First, this avoids aggressive bottlenecks by keeping adjacent widths within a factor of two, which is
consistent with norm-based generalization bounds that depend on products of layer operator norms (large width imbalances tend to inflate these terms) \cite{BartlettFosterTelgarsky2017}.
Second, with our depth choice $L=\lceil(\log n)^{\tau}\rceil$ ($0<\tau<1$) and $m_1\asymp n^{1/6}/\sqrt{\log n}$,
a mild decay ensures the last hidden width remains non-degenerate asymptotically:
\[
m_{L}\;\gtrsim\; \frac{n^{1/6}}{\sqrt{\log n}}\;\rho^{\,L_n-1}
\;\ge\; \frac{n^{1/6}}{\sqrt{\log n}}\;2^{-(\log n)^{\tau}}
\;\xrightarrow[n\to\infty]{}\;\infty,
\]
so the lower truncation (e.g., $m_\ell\ge 2$) becomes inactive for large $n$.
Empirically, architectures that do create very narrow bottlenecks must linearize them to avoid loss of representational
power, which further cautions against overly aggressive shrinkage and
supports a mild decay  \cite{SandlerHowardZhuZhmoginovChen2018}.

Besides, following standard practice, we fix $\tau=\tfrac12$, so that $L=\lceil(\log n)^{1/2}\rceil$. This keeps the depth polylogarithmic, which is suitable well within regimes where $\log n$ depth already
suffices for statistical optimality and approximation (see \cite{FarrellLiangMisra2021,Yarotsky2017}) 
and matches the ubiquitous $\sqrt{\log n}$ factors in nonasymptotic generalization bounds
(e.g., \cite{BoucheronLugosiMassart2013,Wainwright2019}). 


\begin{table}[h]
\centering
\caption{Architectures with $d=10$, $\rho=\tau=0.5$, $K=128$. Values of $p_n/n$ are rounded.}
\label{tab:structural_choices-1}
\begin{tabular}{rcrc}
\toprule
$n$ & $L$ &  Hidden layer widths & $n/p_n$ \\
\midrule
$5\cdot 10^2$ & $3$ &  $10\text{-}5\text{-}3$ & $3$ \\
$10^3$        & $4$ &  $10\text{-}5\text{-}3\text{-}2$ & $5$ \\
$10^4$        & $4$ &  $13\text{-}7\text{-}4\text{-}2$ & $36$ \\
$10^5$        & $5$ &  $17\text{-}9\text{-}5\text{-}3\text{-}2$ & $238$ \\
$10^6$        & $5$ &  $22\text{-}11\text{-}6\text{-}3\text{-}2$ & $1689$ \\
\bottomrule
\end{tabular}
\end{table}

\begin{table}[h]
\centering
\caption{Architectures with $d=37$, $\rho=\tau=0.5$, $K=2030$. Values of $p_n/n$ are rounded.}
\label{tab:structural_choices-2}
\begin{tabular}{rcrc}
\toprule
$n$ & $L$ &  Hidden layer widths & $p_n/n$ \\
\midrule
$5\times 10^3$ & $2$ &  $46\text{-}23$ & $3$ \\
$8\times 10^3$ & $2$ &  $48\text{-}24$ & $4$ \\
$9\times 10^3$ & $2$ &  $49\text{-}25$ & $5$ \\
$8\times 10^3$ & $3$ &  $48\text{-}24\text{-}12$ & $3$ \\
$9\times 10^3$ & $3$ &  $49\text{-}25\text{-}13$ & $3$ \\
$9\times 10^3$ & $4$ &  $49\text{-}25\text{-}13\text{-}7$ & $3$ \\
$10^4$        & $5$ & $49\text{-}25\text{-}13\text{-}7\text{-}4$ & $3$ \\
$10^5$        & $6$ &  $64\text{-}32\text{-}16\text{-}8\text{-}4\text{-}2$ & $19$ \\
$10^6$        & $6$ & $86\text{-}43\text{-}22\text{-}11\text{-}6\text{-}3$ & $120$ \\
\bottomrule
\end{tabular}
\end{table}

\subsection{Learning with additional regularity penalties}\label{sec:learning}

Some of the regularity constraints imposed on the neural network based on the preceding theoretical results may be difficult to enforce directly in practice. An alternative  approach, that let more versatility and expressing power to the network, consists in adding to the cross-entropy loss $L_n(\theta)$ a set of penalization terms that encourage the training procedure to estimate $\theta$ within a subset of $\Theta$ satisfying these constraints \cite{terven2025comprehensive,bischof2025multi}. 
The following paragraphs provide a discussion of these 
regularization penalties.

\subsubsection{Loss function for Assumption \ref{assumption:risk_monotonicity}   }\label{subsection:positive_weigthing}

As mentioned in $\S$ \ref{sec:consistency_survival}, a penalty term can be introduced to encourage the positivity of a subset of the weights. Gupta \textit{et  al.} \cite{gupta2019incorporatemonotonicitydeepnetworks} propose a point‑wise loss for this purpose, which can be adapted to our setting as follows. For each couple $(i,j)$ defining the $j$th mission of the $i$th vehicle, consider the partial gradients, evaluated with automatic differentiation:
\[
g_{i,j,k}^{(\beta)}=\frac{\partial\beta(\theta,\bX^{(i)}_j)}{\partial x^{(i)}_{j,k}},
\quad
g_{i,j,k}^{(\eta)}=\frac{\partial\eta(\theta,\bX^{(i)}_j)}{\partial x^{(i)}_{j,k}}, \quad k\in\{1,\ldots,d\}.
\]
For every sample we measure the violation with respect to the required sign of
each derivative, when $x^{(i)}_{j,k}\in \bX_{o_a}$:
\[
\Delta_i^{(\beta)}(\theta)=
\sum_{x^{(i)}_{j,k}\in \bX_{o_a}}\bigl[-g_{i,j,k}^{(\beta)}\bigr]_+,
\qquad
\Delta_i^{(\eta)}(\theta)=\sum\limits_{i=1}^{n(i)}
\mathbf \1_{\!\bigl\{\beta(\theta,\bX^{(i)}_j)>\beta_0\bigr\}}
\sum_{x^{(i)}_{j,k}\in \bX_{o_a}}\bigl[g_{i,j,k}^{(\eta)}\bigr]_+,
\]
where \([a]_+=\max(0,a)\). This transformation converts only the negative partial derivatives (indicators of monotonicity violations) into a strictly positive penalty proportional to their magnitude, while leaving the contribution null whenever the constraint is fulfilled. Consecutively, the average penalty proposed by \cite{gupta2019incorporatemonotonicitydeepnetworks}, over a  minibatch \(B\subset\{1,\dots,n\}\), writes as
\[
\mathcal L_{\text{mono}}(\theta)=
\frac{1}{|B|}\sum_{i\in B}
\Bigl(
\Delta_i^{(\beta)}(\theta)+\Delta_i^{(\eta)}(\theta)
\Bigr).
\]

\subsubsection{Loss functions for Assumption \ref{assump:distinct_heads}}\label{subsection:orth_cov}

As introduced before, an alternative to the architectural modification for respecting Assumption \ref{assump:distinct_heads} (e.g., to limit the number of free parameters) is: (a) to force the non-collinearity between the two heads (\ref{eq:A7-1}), by minimizing
\begin{eqnarray*}
L_{\text{orth}}(\theta) & = & 
\left(\frac{\langle W_\beta,\,W_\eta\rangle}{ \|W_\beta\|_{2}\,\|W_\eta\|_{2}}\right)^2, 
\end{eqnarray*}
such that this ratio be lower to 1, and: (b) to keep the shared representation $h_{L-1}$ from collapsing onto a low-dimensional subspace ; being implemented using batch covariance, it encourages every coordinate to carry variance and discouraging strong correlations:
{\ttfamily
\begin{enumerate}[leftmargin=10pt,parsep=0pt]
\item Collect the batch activations of the $L-1$ layer:
$\bH =  (h^T_1,\ldots,h^T_m)\in\R^{m \times m_{L-1}}$.
\item Center them:
$\tilde{\bH}  =  H-\mathbf{1}_m\bar{h}^T_m$  with $\bar{h}_m=\frac{1}{m}\sum_{k=1}^m h_i$.
\item Compute the empirical covariance $C=\frac{1}{m-1}\tilde{\bH}^T\tilde{\bH}$
\item Compute the DeCov  regularity \cite{xiong2016regularizing}:
\begin{eqnarray*}
L_{\text{cov}}(\theta) & = & \frac{1}{2}\left(\|C\|^2_F - \|\mbox{diag}(C)\|^2_2\right)
\end{eqnarray*}
\end{enumerate}
}
\noindent which penalizes the crossed covariances. It could be complemented by the OrthoReg penaly \cite{rodriguez2016regularizing} which provides a simpler control on the non-collinearity of weights, without depending on the batch size.  An alternative approach  is the EDM regularization \cite{gu2018regularizing}, which favorizes diversity between groups of unities rather than between all unities.

\subsubsection{Loss function to enforce statistical alignment} \label{subsubsec: loss function alignment}

Because the negative log-likelihood $L_n(\theta)$ is probably highly non-convex once the Weibull
parameters are modelled by neural networks, several distinct couples $(\eta,\beta)$ can achieve
almost the same likelihood; as explained before, global uniqueness is therefore not guaranteed and flat “ridges’’ can appear in the optimisation landscape. To break these plateaus we suggest to supplement $L_n(\theta)$ with a
moment-alignment term that forces the network to reproduce the conditional mean
mission duration. This approach revealed to be successful by \cite{isola2017image,lim2018disease,ren2019deep,zhang2025transformerlsr}, among others.

Let $\widehat T^{(i)}_j(\theta,\bX^{(i)}_j)=\eta(\theta,\bX^{(i)}_j)\,\Gamma\!\bigl(1+\beta(\theta,\bX^{(i)}_j)^{-1}\bigr)$ be the
first Weibull moment, for $i=1,\ldots,N$ and $j=1,\ldots,n(i)$. Given the observations $(z^{(i)}_j)_{i,j}=(\min(t_i,y_i))_{i,j}$, the usual quadratic loss can be written as
\[
\frac{1}{N}\sum_{i=1}^{N}\frac{1}{n(i)}\sum_{j=1}^{n(i)}
\left((z^{(i)}_j-\widehat T^{(i)}_j\left(\theta,\bX^{(i)}_j\right)\right)^2.
\]
However, to account for the presence of right-censoring observations, we modify this statistic by considering  Inverse-Probability-of-Censoring Weights (IPCW ; see details in the Supplementary Material). This technique adjusts the contribution of each observation using weights $w^{(i)}_j$ derived from the Kaplan–Meier estimator of the censoring distribution. These weights  correspond to the jump of the Kaplan–Meier estimator at time $T_i$ \cite{IPCWregression}. Then we can define   
\[
L_{\mathrm{MSE}}(\theta)=\frac{1}{N}\sum_{i=1}^{N}\frac{1}{n(i)}\sum_{j=1}^{n(i)}
w^{(i)}_j\left(z^{(i)}_j-\widehat T^{(i)}_j\left(\theta,\bX^{(i)}_j\right)\right)^2.
\]
This quadratic loss function is convex in $\widehat T_i$ and therefore can act as a
tie-breaker: among likelihood-equivalent solutions it selects the one that aligns the
model with the empirical first moment, improving identifiability and guiding the optimiser out
of flat directions. 

\subsection{Training and testing} 
\label{subsec:performance_comparisons}

As is customary in machine learning workflows, the data are at least divided into training and test sets \cite{hastie2009elements}. In our situation, a \emph{time-based split} is the most appropriate choice. Indeed, since we work with historical sequences, our objective is to dynamically estimate the Weibull parameters through a neural network that captures the temporal evolution of vehicle degradation. Preserving the chronological order of the data is therefore essential. This can be done as follows: for a vehicle $i \in \{1,\ldots,N\}$ and missions $j \in \{1,\ldots,n(i)\}$, we partition the data into:
\begin{itemize}
    \item a \textbf{training set} consisting of $\mathbf{z}^{(i)} = \{z^{(i)}_1, \ldots, z^{(i)}_{n(i-1)}\}$ and $\mathbf{X}^{(i)} = \{\mathbf{X}_1^{(i)}, \ldots, \mathbf{X}_{n(i-1)}^{(i)}\}$;
    \item a \textbf{test set} composed of the pairs $\left(z^{(i)}_{n(i)}, \mathbf{X}_{n(i)}^{(i)}\right)$.
\end{itemize}
Model parameters are classically optimized with the Adam optimizer \cite{kingma2014adam,barakat2021convergence}. On the basis of numerous preliminary experiments, a multi-start approach and several specific features for initializing optimization algorithms were implemented and tested in conjunction with numerical applications in order to promote the success and reproducibility of the estimates. These aspects are described in greater detail in the Supplementary Material. 

The tests are conducted as follows: given the pointwise estimator $\hat{\theta}$ resulting from training,  each set of covariates $\mathbf{X}_{n(i)}^{(i)}$, if $\delta^{(i)}_j=1$ we compute the trained conditional Weibull cumulative distribution function on the latest observation: 
\begin{eqnarray}
\xi^{(i)}_j \ = \ \Pp\left(T \leq z^{(i)}_{n(i)}| \mathbf{X}_{n(i)}^{(i)}\right) & = & 1-\exp\left(-\left(\frac{z^{(i)}_{n(i)}}{\eta\left(\hat{\theta},\mathbf{X}_{n(i)}^{(i)}\right)}\right)^{\beta\left(\hat{\theta},\mathbf{X}_{n(i)}^{(i)}\right)} \right)\label{eq:p-value}
\end{eqnarray}
Beyond checking that such probabilities are not extreme, under the independence assumptions stated in Section~\ref{sec:data-notations}, and under the assumption that the trained model is correct, the sample of $\xi^{(i)}_j$ should follow a uniform distribution on $(0,1)$. 

\section{Numerical experiments}\label{sec:experiments}

This section implements instances of the proposed model. We first conduct a simulation study to illustrate the proper training of the model on data consistent with its assumptions. We then perform a comparative study on the real datasets that motivated this work, where the model is challenged by the DPWTE and WTTE-RNN models, briefly described in the introduction. Technical details on the latter two models are provided in the Supplementary Material.

\subsection{Simulated experiments}\label{subsec:simulations}

In this first analysis framework, we aim to verify the proper behavior of the training procedure when the datasets are consistent with the chosen model. A simulation dataset is named $\mathrm{SIM}-d_a-d_n-n_s-L_s-\Delta$ 
where $d_a$ is the dimension, $d_n$ is the number of numerical covariates, 
$n_s$ is the total number of missions, $L_s$ is the number of hidden layers and $\Delta$ is the overall censoring rate. The covariates are chosen among the true known covariates for Fleet \#1 (see Table \ref{tab:data.description}) to ensure physical relevance through time of the simulated datasets.  The simulation values for 
Weights $\bW$  and biases $\bb$ were chosen after preliminary tests, such that the resulting Weibull simulations are not irrelevant with the observed durations described in Section \ref{sec:data-notations}. In Algorithm 1,  they are simulated with Gaussian distributions. The unusual disparity in magnitude between the Gaussian parameters arises from the deliberate omission of batch-normalization (BN) steps within the network layers. As detailed in the Supplementary Material (\S 2.2), incorporating BN tends to bias training toward reducing the RMSE, whereas its absence allows the optimization to emphasize the contrast between the parameters $(\eta,\beta)$, which is the primary objective.

\begin{algorithm}[H]\label{algo:theta_simulation}
\caption{Parameter structure simulation}
\begin{algorithmic}
\Require{$d_a, d_n, n_s, L_s \rightarrow \ \text{selected covariates} \ \bX$}
\begin{enumerate}
\item Given $n_s$ and $L_s$, select an appropriate architecture according to Tables \ref{tab:structural_choices-1}-\ref{tab:structural_choices-2}
\item Sample $\theta_{\text{\tiny sim}}=(\bW,\bb)$ as $\left\{\begin{array}{lll}
\bW & \sim & {\cal{N}}_{{>0}}(0.1,0.1) \ \ \ \ \text{(positive truncated Gaussian)} \\
\bb & \sim & {\cal{N}}(10,5). 
\end{array}\right.$
\end{enumerate}
\end{algorithmic}
\end{algorithm}

\noindent The number $N_s$ of vehicles is chosen as $N_s= N \frac{n_s}{n}$ where $n$ is the number of true vehicles in Fleet \#1 (see Section \ref{sec:data-notations}). Then, given $N_s$, the numbers $n_{sm}(i)$ of simulated mission per vehicle are randomly selected using a balanced multinomial distribution
\begin{eqnarray*}
n_{sm}(1),\ldots,n_{sm}(N_s) & \sim & {\cal{M}}\left(n_s,\left(\frac{1}{N_s},\ldots,\frac{1}{N_s}  \right)\right).
\end{eqnarray*}
Finally, the right-censoring values are randomly distributed and chosen as $\tilde{\alpha}=90\%$ quantiles of each selected conditional Weibull distribution. This process leads to the following dataset simulation algorithm.  

\begin{algorithm}[H]
\caption{Dataset simulation}
\begin{algorithmic}
\Require $\theta_{\text{\tiny sim}}$, $N_s$, $\Delta$, $\{n_{sm}(1),\ldots,n_{sm}(N_s)\}$, $\bX$
\For{$i = 1,\ldots,N_s$}
  \For{$j = 1,\ldots,n_{sm}(i)$}
  \State \textbf{Sample} $t^{(i)}_j|\theta_{\text{\tiny sim}}, \bX^{(i)}_j \;\sim\; {\cal{W}}\left(\eta \ \!\big(\theta_{\text{\tiny sim}}, \bX^{(i)}_j\big) \ , \ \beta \ \!\big(\theta_{\text{\tiny sim}}, \bX^{(i)}_j\big) \right)$
    \State \textbf{Select} randomly a censored vehicle tra

    £jectory:
    \State Sample uniformly $U_{i,j} \sim \mathcal{U}[0,1]$
    \If{$U_{i,j} < \Delta$}
      \State Set $\delta^{(i)}_j \gets 0$ and replace $t^{(i)}_j$ by
      $
      y^{(i)}_j
      \;=\;
      \eta \ \!\big(\theta_{\text{\tiny sim}}, \bX^{(i)}_j\big)\,
      \left[-\log\!\big(1-\tilde{\alpha}\big)\right]^{
        1/\beta \ \!\big(\theta_{\text{\tiny sim}},\bX^{(i)}_j\big)}
      $
    \Else
      \State Set $\delta^{(i)}_j \gets 1$
    \EndIf
  \EndFor
\EndFor
\end{algorithmic}
\end{algorithm}

A summary of simulated datasets is  provided in Table \ref{tab:simulated_config_results}. These choices yield  simulations that are reasonably realistic relative to the empirical data, with an median $L_1$ relative error of 10.56\% with respect to these empirical data (Figure \ref{fig:comparison-simulation}). 




\begin{figure}[hbtp]
    \centering
    \includegraphics[width=0.7\linewidth]{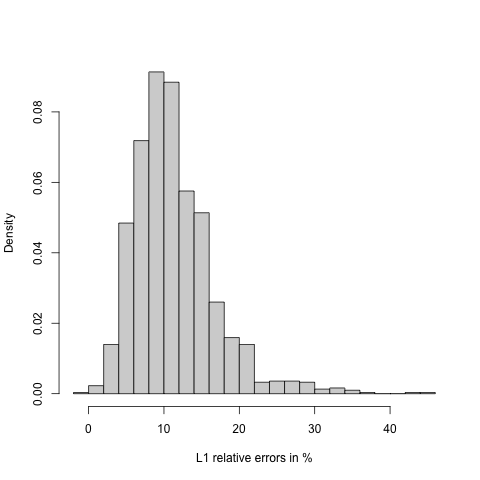}
    \caption{Distribution in percentage of  $L_1$ relative errors  between simulated and real mission durations.} 
    \label{fig:comparison-simulation}
\end{figure}



For a simulated parameter $\theta_{\text{\tiny sim}}$ and  corresponding replicated simulated datasets $\text{SIM}_k-d_a-d_n-n_d-L-\Delta$, for $k=1,\ldots,K$, the accuracy properties of the statistical estimates $\{\hat{\theta}^{(k)}\}_k$ are summarized by two macroscopic, complementary indicators: 
\begin{enumerate}
\item The symmetric median relative error across replications, or symmetric median absolute percentage (sMdAPE) \cite{hyndman2006another},
\begin{eqnarray}
\mathrm{sMdAPE}_{\text{med}}=\mathrm{median}_{\,k=1..K}\ \mathrm{sMdAPE}^{(k)} \label{eq:sMdAPE}
\end{eqnarray}
where
\[
\mathrm{sMdAPE}^{(k)} \;=\;
\begin{cases}
\displaystyle \frac{\,\|\hat\theta^{(k)}-\theta_{\text{\tiny sim}}\|_2}{\|\hat\theta^{(k)}\|_2+\|\theta_{\text{\tiny sim}}\|_2}, & \text{if } \|\hat\theta^{(k)}\|_2+\|\theta_{\text{\tiny sim}}\|_2>0,\\[9pt]
0, & \text{otherwise.}
\end{cases}
\]
The statistic (\ref{eq:sMdAPE}) is dimensionless, bounded in $[0,1]$, well defined when vectors contain zeros, more robust than the (symmetric) MAPE, and quantifies the typical relative magnitude error of the estimator.
\item The median directional alignment (median cosine similarity) defined by 
\begin{eqnarray}
\cos_{\text{med}}=\mathrm{median}_{\,k=1..K}\ \cos\phi^{(k)}, \label{eq:mcosin}
\end{eqnarray}
where 
\[
\cos\phi^{(k)} \;=\;
\begin{cases}
\displaystyle \frac{\hat\theta^{(k)}\!\cdot\theta_{\text{\tiny sim}}}{\|\hat\theta^{(k)}\|_2\,\|\theta_{\text{\tiny sim}}\|_2}, & \text{if } \|\hat\theta^{(k)}\|_2>0 \text{ and } \|\theta_{\text{\tiny sim}}\|_2>0,\\[10pt]
1, & \text{if } \|\hat\theta^{(k)}\|_2=\|\theta_{\text{\tiny sim}}\|_2=0,\\[4pt]
0, & \text{otherwise.}
\end{cases}
\]
Again, the statistic (\ref{eq:mcosin}) is dimensionless and robust to zero vectors.  It lies in $[-1,1]$ and measures the directional agreement between estimated and true (simulation) parameters.
\end{enumerate}
However, since this estimation is carried out in high dimension and may suffer from the lack of underlying convexity (resulting, in particular, in non-uniqueness of the estimated weights), it is also relevant to compute statistic~(\ref{eq:p-value}) in order to assess whether the training of the model yields an estimation sufficiently accurate to account for the most recent mission duration values generated. 
Since we aim to account for the full predictive uncertainty, the statistic~(\ref{eq:p-value}) is in practice replaced by $\E_{\eta,\beta}[\xi^{(i)}_j]$ 
where the expectation is taken with respect to the Monte Carlo Dropout (MDC) distribution of  $\eta$ and $\beta$, used to estimate predictive uncertainty (See details on Supplementary Material): 
\begin{eqnarray*}
\E_{\eta,\beta}\left[\xi^{(i)}_j\right] & \simeq &   1 - \frac{1}{M}\sum\limits_{m=1}^M\exp\left(-\left(\frac{z^{(i)}_{n(i)}}{\eta_m\!\left(\hat{\theta},\mathbf{X}_{n(i)}^{(i)}\right)}\right)^{\beta_m\!\left(\hat{\theta},\mathbf{X}_{n(i)}^{(i)}\right)}\right)
\end{eqnarray*}
where $(\eta_m,\beta_m)_{k=1,\ldots,M}$ are simulated by MCD. Results are aggregated using the median over all missions $(i,j)$. They are summarized on Table \ref{tab:simulated_config_results}. The observed rate of monotonicity violations after training between specific covariates and survival, measured through negative partial derivatives in ${\cal{L}}_{\text{mono}}$ (see $\S$ \ref{subsection:positive_weigthing}), was found to be lower than 3\% for all experiments. \\




\begin{table}[H]
    \centering
    \caption{Simulation dataset configurations and estimation results, based on $K=20$ replicated datasets for each 
    $\text{SIM}-d_a-d_n-n_s-L-\Delta$ configuration.  The censoring rate is $\Delta=5\%$ for all simulated datasets. $M=1000$ Monte Carlo 
    Dropout (MCD) simulations are used for estimating 
    $\E_{\eta,\beta}[\xi^{(i)}_j]$.  Results (three rightmost columns) are averaged on the $K$ datasets and the standard 
    deviation is provided between parentheses.}
    \label{tab:simulated_config_results}
    \begin{tabular}{llllllll}
        \hline
         Dim. $d_a$ & Nb. ordinal  & Mission  & Nb. hidden  &  $\mathrm{sMdAPE_{med}}$ & $\cos_{med}$  &$\text{med}_{i,j}\E_{\eta,\beta}[\xi^{(i)}_j]$ \\
        & covariates $d_n$ &  number $n_s$ & layers $L_s$ &  \\
        \hline
         14 & 2 & 8, 900  & 5 & 0.22 (0.09)  & 0.85 (0.12) & 0.72 ($2. 10^{-4}$) \\
         8 & 2 & 8, 900  & 5 & 0.17 (0.04) & 0.77 (0.08) & 0.58 ($10^{-4}$) \\
         5 & 2 & 95, 000  & 5 & 0.21 (0.03) & 0.73 (0.13) & 0.63 ($3.10^{-4}$) \\
         8 & 2 & 95, 000 & 5 & 0.28 (0.05) & 0.87 (0.18) & 0.47 ($9. 10^{-5}$) \\
         5 & 2 & 95, 000 & 4 & 0.14 (0.04) & 0.82 (0.14) & 0.66 ($10^{-4}$) \\
         5 & 2 & 1, 700 & 4 & 0.39 (0.04) & 0.78 (0.11) & 0.88 ($2. 10\mathrm{e}^{-4}$) \\
        \hline
    \end{tabular}
    \label{tab: simulation configuration_results}
\end{table}

The results illustrate that in this idealized situation, the training of WTNN with coherent data generally lead to accurate results allowing relevant predictions. The sample size relative to the number of covariates, together with the architectural choices, have a significant impact on the model's ability to be correctly estimated from coherently generated data. Adequate estimation corresponds to a pair $(\mathrm{sMdAP}, \cos)$ close to $(0,1)$. Larger sample sizes and lower model dimensionality yield better estimation. These ideal values are never attained, indicating a likely limited full identifiability of the model. This outcome is expected, as it is linked to the presumed lack of convexity of the global loss function used for training. The predictive ability of these models remains reasonable, with the values of $\mathrm{med}_{i,j}\,\mathbb{E}_{\eta,\beta}[\xi^{(i)}_j]$ lying in the non-extreme regions of $(0,1)$.



\subsection{Experiments conducted on real datasets}

We compare here the training and predictive 
performances of WTNN, DPWTE, and WTTE-RNN (introduced in the Introduction and detailed in the Supplementary Material) on three real-world military vehicle trajectory datasets described in Table \ref{sec:consistency_survival}, using the metrics defined in Section~\ref{subsec:performance_comparisons}.  
\subsubsection{Performance indicators}

Several performance indicators are 
employed 
to compare the predictive ability of different instances of the competing models (including variations of the WTNN model). £
They are selected according to their ability to assess both discrimination and calibration.
\begin{itemize}
\item \emph{Discrimination} refers to the model's ability to correctly distinguish and classify different categories. For instance, it evaluates the capacity to correctly identifying  whether a vehicle will successfully complete its mission or not (class 0 or 1).
\item \emph{Calibration}, on the other hand, assesses how well the predicted probabilities align with observed outcomes. In other words, it compares predicted versus real probabilities.
\end{itemize}
Among the most widely used performance indicators used in both medical \cite{PerformanceIndicator, performanceIndicatorSante1, performanceIndicatorSante2} or mechanical applications \cite{DPWTE,DeepHit}, we focus on the following ones (see \cite{naidu2023review} for a review).  For each vehicle $i$ and mission $j$, let us broadly denote $s^{(i)}_j$ the prognostic score such that a \emph{higher} score indicates a \emph{higher risk}
(and therefore a shorter expected time between unavaibilities (or survival time)). In our situation, we use the opposite of the mean predictor $s^{(i)}_j =-\widehat T^{(i)}_j(\theta,\bX^{(i)}_j)$. 
\begin{enumerate}
\item {\bf Time-dependent AUC.} 
The AUC (Area Under the Curve) is well known in machine learning in connection with the ROC
(Receiver Operating Characteristic). The ROC curve plots the true positive rate against the false
positive rate, providing a visual tool to assess whether a model is able to distinguish true positives
from true negatives \cite{bradley1997use}. The area under this curve has the probabilistic
interpretation of the likelihood that the model correctly ranks a randomly chosen positive
observation higher than a negative one. In the context of survival analysis, the AUC becomes {time-dependent}. At a given time $t$,
it quantifies how well the model separates missions that fail before $t$ from those that survive
beyond $t$:
\[
\mathrm{AUC}(t)
=
\mathbb{P}\big( s^{(i)}_j > s^{(k)}_l \;\big|\; z^{(i)}_j < t,\ \delta^{(i)}_j=1;\ z^{(k)}_l > t \big),
\]
that is, the probability that a mission $(i,j)$ will fail before $t$ is assigned a higher risk score
than a mission $(k,l)$ that can still be completed at time $t$.

\item The {\bf C-index}, or {\bf concordance index} \cite{pinto2015outlier}, evaluates discrimination by generalizing the AUC. It measures the proportion of all comparable pairs in which predictions and observed outcomes are concordant: the set of \emph{comparable pairs} is defined as
\[
\mathcal{P} \;=\; \{\,((i,j),(k,l)) \,:\, z^{(i)}_j < z^{(k)}_l \ \text{and}\ \delta^{(i)}_j= 1 \,\}.
\]
That is, mission $(i,j)$ must have experienced a failure event before mission $(k,l)$,
which makes the comparison informative. (Pairs with $z^{(i)}_j = z^{(k)}_l$ or where the earliest mission is censored are non-comparable.)  For predictors trained on censored datasets, the empirical C-index based on Harrell's formulation \cite{Harrell, concordanceindex} is given by 
\[
\widehat{C}
\;=\;
\frac{1}{\lvert \mathcal{P} \rvert}
\sum_{((i,j),(k,\ell))\in \mathcal{P}}
\Big[
\1_{\!\{\,s_{i,j} > s_{k,\ell}\,\}}
\;+\;
\tfrac{1}{2}\,\1_{\!\{\,s_{i,j} = s_{k,\ell}\,\}}
\Big].
\]
It satisfies $\widehat{C} \in [0,1]$, with $\widehat{C}=1$ for perfect discrimination
$\widehat{C}\approx 0.5$ for random (non-informative) rankings, and $\widehat{C}\approx 0$ for worst discrimination. 
\item {\bf Integrated Brier score.} A proper scoring rule, the Brier Score (BS) \cite{BrierArticle} can be used to  evaluate the accuracy of the predicted survival function at a given time $t$ \cite{performanceIndicatorSante1}. Roughly speaking, it answers to the calibration question "how close to the real probabilities are our estimates?".  For each vehicle $i$ and mission $j$, denote $\hat{S}(t\,|\,\bX^{(i)},j,\theta)$ the predicted survival
probability at time $t$ for covariates $\bX^{(i)}$. Lying between 0 (best possible
value) and 1, the BS is defined in our censored context by
 \begin{equation}
   \mathrm{BS}(t) = \frac{1}{n}  \sum_{i,j}  w^{(i)}_j  \left( \1_{\{ z^{(i)}_j > t \}}  - \hat{S}(t\,|\,X^{(i)}_j,\theta)\right)^2,
\end{equation} 
reminding that $n$ is the total number of missions across all vehicles. The Integrated BS (IBS; \cite{graf1999assessment})
\begin{equation*}
    \mathrm{IBS} = \frac{1}{t_\text{max}}\int_{0}^{t_\text{max}} \mathrm{BS}(t) \ d t 
\end{equation*}
allows to assess performance over the time interval $[0, t_{\text{max}}]$.
\end{enumerate}

\subsubsection{Comparisons}

As the DPWTE model is based on a Weibull mixture, several instances of this model were used for this comparative study. The \texttt{DPWTE p1} configuration indicates that the model is trained using a single Weibull distribution, making it directly comparable to the WTTE and WTNN models. With \texttt{DPWTE p10}, the model selects the optimal number of Weibull components from a range of 1 to 10, which typically results in a mixture of 3 to 5 components in our study, as highlighted in Table \ref{tab:Performance_indicator_synthesis}. In this table, the bold values correspond to the best performance for each indicator and each fleet. The bold model names indicate the most performant models for each fleet with respect to a specific indicator. The diversity of performance indicators enables a more nuanced model selection process. Given that the primary objective is to achieve well-calibrated predictions, the IBS appears to be the most appropriate metric. The concordance index and AUC can serve as complementary criteria to support the final decision. Remind that the concordance index only reflects a model’s ability to rank survival times correctly; it does not assess calibration, probability accuracy, or temporal coherence, and therefore cannot, on its own, indicate whether the predicted survival distribution is meaningful.\\





\begin{table}[H]
\centering
\caption{Performance indicators computed for the models in competition over three different vehicle fleets (described on Table \ref{sec:consistency_survival}).}
\label{tab:Performance_indicator_synthesis}
\renewcommand{\arraystretch}{1.2} 
\begin{tabular}{lllll} 
Performance indicator &
 Model  & 
 Fleet 1 & Fleet 2 & Fleet 3 \\
\hline

\multirow{4}{*}{C-Index} & 
                     DPWTE p10 & 0.51 & 0.50 & $ 0.497$ \\ 
                     & DPWTE p1  & 0.50 & 0.49 & $0.508 $\\ 
                     & \textbf{WTTE}  & $ \mathbf{0.58}$ & $\mathbf{0.80} $& $ \mathbf{0.518}$ \\ 
                     & WTNN  & 0.44 & 0.44 & 0.506 \\ 
\\

\multirow{4}{*}{IBS} &  
                     DPWTE p10 & 0.34 & 0.35 & 0.28 \\ 
                    & DPWTE p1 & 0.52 & 0.52 & 0.41 \\ 
                    & \textbf{WTTE} & $ \mathbf{0.04} $ & $ \mathbf{0.05} $ & 0.30 \\ 
                    & \textbf{WTNN} & $ \mathbf{0.04} $ & $ \mathbf{0.05} $ & $ \mathbf{0.10} $\\ 
\\

\multirow{4}{*}{Time-dependent AUC} & 
                    DPWTE p10 & 0.48  & 0.50  & $ \mathbf{0.51}$  \\ 
                    & DPWTE p1 & 0.52 & 0.52 & 0.504 \\ 
                    & WTTE & 0.37 & 0.05 & 0.491 \\ 
                    & \textbf{WTNN} & $ \mathbf{0.57}$ & $ \mathbf{0.58}$ & $\mathbf{0.51}$ \\ 
\\


\multirow{4}{*}{Training time (s)} &  
                    \textbf{DPWTE p10} & $ \mathbf{6}$ & $\mathbf{2}$  & $\mathbf{2}$  \\ 
                    & \textbf{DPWTE p1} & $\mathbf{6}$ & $\mathbf{2}$ & 3 \\ 
                    & WTTE & 2 399 & 55 & 100 \\ 
                    & WTNN & 37 & 9 & $\mathbf{2}$ \\ 
\hline

\end{tabular}
\end{table}

These experiments highlight that the concordance index alone gives a distorted view of performance: WTTE achieves the highest C-index on two fleets, but its calibration is poor and its training time becomes prohibitive for larger datasets. DPWTE performs inconsistently, with mixtures that slightly improve discrimination but fail to produce well-aligned survival probabilities. When calibration is examined using the integrated Brier score, WTNN consistently equals or outperforms WTTE and clearly outperforms all DPWTE variants. This indicates that WTNN not only ranks missions reasonably well, but more importantly, assigns survival probabilities that remain consistent with observed outcomes over time. WTNN also maintains competitive AUC values, showing that its discrimination does not collapse despite the additional structural constraints.

The stability of the WTNN's learning behavior is another interesting advantage. Its optimization remains fast and reproducible across all fleets, while the WTTE becomes unreliable and significantly slower as the sample size increases. As a result, based on this data, WTNN appears to be the better model that offers calibrated predictions, acceptable discrimination, controlled computational cost, and robustness on heterogeneous real-world fleets.





An additional model comparison, in terms of predictive power, consists in assessing how well the models can predict the survival value for the last observation across vehicle trajectories by estimating (\ref{eq:p-value}). Figure \ref{fig:Boxplot_interval} illustrates, for WTNN, how the  survival function behaves when integrated over the predictive uncertainty quantified through MCD. Table \ref{tab:pred_pvalues} then summarizes the quantile thresholds corresponding to the position of each last observation within this survival distribution. These results show that the final observations are consistently better explained by the WTTE and WTNN models, in the sense that they correspond to non-extreme quantile levels. Moreover, WTNN provides greater robustness than WTTE, as reflected by the narrower range of these quantile thresholds.

\begin{figure}[H]
    \centering
    \includegraphics[width=0.85\linewidth]{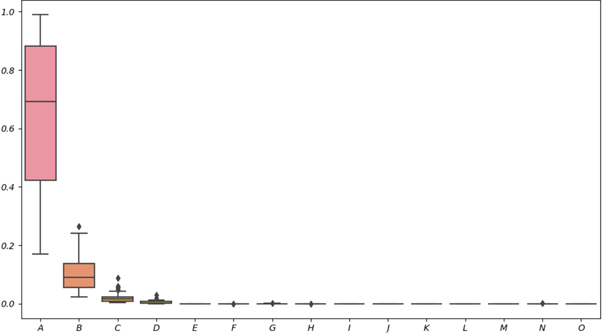}
    \caption{Predictive survival for the WTNN model (estimated for Fleet \# 3).}
    \label{fig:Boxplot_interval}
\end{figure}

\begin{table}[H]
    \centering  
     \caption{Positions (quantile thresholds) within their predictive distributions of last observations per model and fleet. Results are provided as 5\%-95\% intervals over all vehicle missions.}
    \label{tab:pred_pvalues}
\begin{tabular}{llll}
\hline
Model & Fleet 1 & Fleet 2 & Fleet 3 \\
\hline
DPWTE p10 & 2.4\% - 99.8\% & 0.1\% - 99.6\% & 1.1\% - 99.8\% \\
DPWTE p1 & 3.1\% - 96.7\% & 2.5\% - 97.4\% & 2.1\% - 97.2\% \\
WTTE & 18.5\% - 89.3\% & 12.9\% - 92.6\% & 5.8\% - 94.2\%\\
WTNN & 28.9\% - 76.9\% & 23.3\% - 83.1\% & 12.7\% - 88.3\%\\
\hline
\end{tabular}
\end{table}

\section{Discussion}

This work is motivated by the development of an effective lifetime model for military vehicles, based on covariates observed at the system level between successive immobilization periods. Within this censored framework, we introduce a Weibull-based model conditioned on a two-output neural network, termed WTNN, which can be viewed as a generalization of accelerated lifetime models. The network architecture was designed to reflect qualitative prior knowledge on the relationship between certain covariates and Weibull survival. A theoretical analysis further established architectural constraints ensuring sieve consistency. Numerical experiments, both on simulated and real datasets, indicate that enforcing coherence between the regression function and the Weibull structure allows WTNN to perform competitively with approaches originally proposed in the medical survival literature, while adapting to industrial constraints. Theoretical results also show that the problem is statistically well-posed, supporting the use of this method in settings where the sample size $n$ may vary, and in particular, can become much larger.

While we have aimed to propose construction strategies more suited than those found in the existing literature, the technical contribution presented here remains open to refinement through more extensive numerical experiments. These should include both replication of simulation-based studies, an improvement of censored data sampling (for instance following \cite{ramos2024sampling}) and a deeper analysis of training and validation procedures. The architectural insights derived from the theoretical developments of this paper can further guide such investigations. Notably, the theoretical framework can be readily extended to other conditional lifetime models, such as Cox, lognormal, or competing risks models which we discuss in more detail below.  

In a broader perspective, the lack of convexity of the full loss function remains a major challenge in learning, as it may hinder the convergence of optimization algorithms. Potential directions to improve robustness could build on recent advances in the interpretability of these algorithms \cite{bensaid2024convergence}. Addressing this non-convexity may be combined with the need to account for uncertainty through a Bayesian approach. Such a methodology would additionally yield Bayesian estimators of the predictive survival function endowed with a total order \cite{robert2007bayesian}, thereby enabling an effective ranking of vehicles, one of the original motivations of this work. This line of research will be the focus of future contributions.  

Regarding interpretability, it is clear that an essential requirement for the industrial deployment of WTNN, based on a dense neural network, is to provide mechanisms for understanding the network’s behavior \cite{li2022interpretable}. This concern also motivated the theoretical study presented here. Although not directly addressed in this work, we consider it essential to pursue such investigations in future contributions. In particular, we believe that the possibility of using Proposition~\ref{prop:constraints} to interpret monotonicities detected in Weibull survival represents a key strength of our approach, readily transferable to other models, and a first step toward practical interpretability, as stressed in \cite{nguyen2023mononet}. Deeper exploration of architectural choices and their connection to classical reliability quantities should also shed light on potential improvements to the proposed framework.  

Additional avenues for future research to enhance the generalization of our approach include:  
(a) relaxing the strong AGAN assumption by adopting more general frameworks, such as virtual age models (see Introduction), which allow for the explicit modeling of maintenance effects;  
(b) replacing Weibull models with competing risks models (e.g., \cite{bousquet2006alternative}). Unlike mixture-based approaches such as DeepWeiSurv~\cite{bennis2020estimation} and DPWTE~\cite{bennis2021dpwte}, competing risks models provide interpretability by explicitly associating each risk with a distinct failure mode or degradation mechanism~\cite{kalbfleisch2002statistical}, making them well-suited to operational contexts where subsystem-specific or usage-dependent failures are expected. This direction could leverage recent advances in strictly proper scoring rules for censored competing-risks survival data~\cite{alberge2025p52}, which enable consistent probabilistic modeling while supporting stochastic optimization over mini-batches. Building on the core ideas of this paper, and given the high cost and scarcity of industrial data, we suggest developing a shared neural backbone for the parameters of such competing risks models, including classification probabilities. These directions are left for future work.

\bibliographystyle{plain}
\bibliography{biblio}

\appendix
\section{Proofs}

\begin{preuve}[Proposition \ref{prop:constraints}]
Assumption (\ref{hypo:monotonicity_2}) is equivalent to have $\log p(t|\mathbf{X})\downarrow \mathbf{X}^{o_a}$ $\forall t\geq 0$, and equivalently
\begin{eqnarray*}
\hat{\beta}(\theta,\mathbf{X})\left(\log t - \log \hat{\eta}(\theta,\mathbf{X})\right) \uparrow \mathbf{X}^{o_a} \ \ & & \ \forall \theta, \ \forall t\geq 0.
\end{eqnarray*}
In particular, this is true for any $t=\alpha\hat{\eta}(\theta,\mathbf{X})$ with $\alpha>1$. Hence $\hat{\beta}(\theta,\mathbf{X})\uparrow \mathbf{X}^{o_a}$. Furthermore, (\ref{hypo:monotonicity_2}) implies that 
\begin{eqnarray*}
\mathbb{E}[T|\mathbf{X}] & = & \int_{0}^{\infty} p(t|\mathbf{X}) \ dt\ \downarrow \mathbf{X}^{o_a} \\
& = & \hat{\eta}(\theta,\mathbf{X})\Gamma\left(1+1/\hat{\beta}(\theta,\mathbf{X})\right).
\end{eqnarray*}
Consequently, it is necessary that $\hat{\eta}(\theta,\mathbf{X})\downarrow \mathbf{X}^{o_a}$ when $\Gamma(1+1/\hat{\beta}(\theta,\mathbf{X}))\uparrow \mathbf{X}^{o_a}$, namely when $\hat{\beta}(\theta,\mathbf{X})>\beta_0$. $\square$
\end{preuve}


\begin{preuve}[Proposition \ref{prop:upper_bound_eta}]
From the properties of the \texttt{softplus} activation function, for any $\ell>1$:
\begin{eqnarray*}
\|H_{\ell}\|_{\infty} & = & \bigl\|\text{\texttt{softplus}}\bigl(\bW_{\ell}H_{\ell-1}+b_{\ell}\bigr)\bigr\|_{\infty}, \\
& \leq & \log 2 + \|\bW_{\ell}H_{\ell-1}+b_{\ell}\|_{\infty}, \\
& \leq & \log 2 + \sqrt{m_{\ell-1}} R \|H_{\ell-1}\|_{\infty} + \|b_{\ell}\|_{\infty} \ \ \ \text{from Frobenius inequality,} \\
& \leq & \sum\limits_{k=0}^{\ell-1} \left(\prod\limits_{j=1}^k m_{\ell-j}\right)^{1/2}\left(R^k\log 2 + \|b_{k}\|_{\infty}\right) + \left(\prod\limits_{j=1}^{\ell} m_{\ell-j}\right)^{1/2}\|\bX\|_{\infty}
\end{eqnarray*}
with $m_0=d$ and setting $b_0=0$ by convention. With $\|b_k\|_{\infty}\leq \|b_k\|\leq R$, it comes, from (\ref{eq:eta_predictor})
\begin{eqnarray*}
\|\tilde{\eta}(\theta,\bX)-\eta_{min}\|_{\infty} & \leq & \log 2 + R\sum\limits_{k=0}^{L} m^{k/2}_{\sup}\left(1+R^{k-1}\log 2\right) +  m^{L/2}_{\sup}\|\bX\|_{\infty}
\end{eqnarray*}
which proves the results, since this upper bound is not modified by the final pooling operation.
\end{preuve}


\begin{preuve}[Proposition \ref{prop:C2-diff}]
The proof of this proposition follows from that of the more general Proposition~\ref{prop:C2-networks} stated below.
\end{preuve}


\begin{proposition}\label{prop:C2-networks}
Denote generically $\sigma$ the activation function and assume  $\sigma$ of class $C^{2}$ satisfying 
$\|\sigma'\|_{\infty}\le S_{1}$ and $\|\sigma''\|_{\infty}\le S_{2}$ (Assumption \ref{assumption:regularity_activations}). Denote $\nabla_{\theta}$ and $\nabla_{\theta,\theta}$ the gradient and Hessian operators with respect to $\theta$. 
Under Assumptions~\ref{assumption:compacity_param} and~\ref{assumption:compacity_input}, 
for any scalar network output $g(\theta,\mathbf X)$, in particular for $g(\theta,\mathbf X)=\beta(\theta,\mathbf X)$ and $g(\theta,\mathbf X)=\eta(\theta,\mathbf X)$,
\begin{enumerate}[label=(\alph*),nosep]
  \item $g(\theta,\mathbf X)$ is $C^{2}$ with respect to~$\theta$;
  \item $\displaystyle\sup_{\theta\in\Theta_{n}}\bigl\|\nabla_{\theta}g(\theta,\mathbf X)\bigr\|
        \le \Delta_{1}\,R^{L}$; \\`
\vspace{-0.2cm}

  \item $\displaystyle\sup_{\theta\in\Theta_{n}}\bigl\|\nabla^{2}_{\theta\theta}g(\theta,\mathbf X)\bigr\|
        \le \Delta_{2}\,R^{2L}$, \\
\end{enumerate}
\vspace{-0.2cm}

\noindent where the positive constants $\Delta_{1},\Delta_{2}$ depend only on  $R$, $S_{1}$, $S_{2}$ and~$\Delta$. Note that if activation functions are \texttt{softplus}, their derivatives are sigmoid, then $S_1<1$, and $\Delta_1>S_1\Delta/(R(1-S_1))$.
\end{proposition}

\begin{proof}[Proof of Proposition~\ref{prop:C2-networks}]
We prove by induction on the hidden layers $\ell = 1,\dots,L$ that
\[
  \bigl\|\nabla_{\theta} h_\ell(\theta,\mathbf X)\bigr\|
      \;\le\; \Delta_{1}\,R^{\ell},
  \qquad
  \bigl\|\nabla_{\theta\theta}^{2} h_\ell(\theta,\mathbf X)\bigr\|
      \;\le\; \Delta_{2}\,R^{2\ell},
\]
where $h_\ell$ denotes the output of the $\ell$-th layer. 

\begin{description}
\item[Step 0 (input).]  
The input layer is $h_0(\theta,\mathbf X)=\mathbf X$, which is independent of~$\theta$.
Hence $J_0:=\nabla_{\theta}h_0=0$ and
$H_0:=\nabla_{\theta\theta}^{2}h_0=0$.

\item[Induction hypothesis.]  
Assume for some $\ell-1\ge 0$ that
\(
\|J_{\ell-1}\|\le \Delta_{1}R^{\ell-1}
\)
and
\(
\|H_{\ell-1}\|\le \Delta_{2}R^{2(\ell-1)}.
\)
Set
\( z_{\ell}=W_{\ell}h_{\ell-1}+b_{\ell}\)
and
\(h_\ell=\sigma(z_\ell)\).
With $\otimes$ denoting the Kronecker product, the chain rule yields
\[
  J_{\ell}
  \;=\;
  \sigma'(z_\ell)\Bigl[
         (I\!\otimes\! h_{\ell-1}^{\top})
         \frac{\partial\operatorname{vec}W_{\ell}}{\partial\theta^{\top}}
         \;+\;
         W_{\ell}J_{\ell-1}
     \Bigr],
\]
where $\operatorname{vec}$ stacks the columns of $W_{\ell}$ and
$\partial\operatorname{vec}/\partial\theta^{\top}$ is the corresponding Jacobian.
Hence
\begin{eqnarray*}
  \|J_{\ell}\|
  & \leq & 
  S_1\bigl(\|h_{\ell-1}\| + \|W_\ell\|\,\|J_{\ell-1}\|\bigr), \\
  & \leq & 
  S_1\bigl(\Delta R^{\ell-1}+R \Delta_{1} R^{\ell-1}\bigr), \\
  & \leq & 
  \Delta_{1}R^{\ell},
\end{eqnarray*}
provided $\Delta_{1}\ge S_1(\Delta+\Delta_{1})$ (remind that $R>1$).

For the Hessian,
\[
  H_{\ell}
  \;=\;
  \sigma''(z_\ell)\bigl[J_{\ell}J_{\ell}^{\top}\bigr]
  \;+\;
  \sigma'(z_\ell)
     \Bigl[
        H_{\ell-1}\otimes W_\ell
        \;+\;
        (J_{\ell-1}\otimes I)
          \frac{\partial\operatorname{vec}W_{\ell}}
               {\partial\theta^{\top}}
        \;+\;\text{sym.}
     \Bigr],
\]
where “\text{sym.}” denotes the transposed term, i.e.\ $A+\text{sym.}=A+A^{\top}$.
Taking operator norms gives
\begin{eqnarray*}
  \|H_{\ell}\|
  & \leq & 
  S_2\|J_{\ell}\|^{2}
  \;+\;2S_1\|W_\ell\|\,\|H_{\ell-1}\|, \\
  & \leq & 
  S_2\Delta_{1}^{2}R^{2\ell}
  \;+\;2S_1\Delta_{2}R^{2\ell}, \\
  & \leq & 
  \Delta_{2}R^{2\ell},
\end{eqnarray*}
for any $\Delta_{2}\ge S_2\Delta_{1}^{2}+2S_1\Delta_{2}$.
\end{description}
Because $\beta(\theta,\mathbf X)$ is affine in $h_{L-1}$, its derivatives inherit the same bounds.
For $\eta(\theta,\mathbf X)$, the chain rule yields identical bounds up to a multiplicative constant.  
Thus, by induction, inequalities~(b)–(c) hold for every layer. Finally, the $C^{2}$ property of $g(\theta,\mathbf X)$ follows from the fact that the operations applied at each layer (matrix multiplication, addition, and composition with $\sigma$) all preserve the $C^{2}$ class. Consequently, if $h_{\ell-1}$ is $C^{2}$, then $h_{\ell}$ is $C^{2}$ as well. The inductive argument thus implicitly establishes property~(a), which 
completes the proof. \qedhere
\end{proof}


\begin{preuve}[Corollary \ref{cor:existence_consistency}]
The continuity of the mapping $\theta \mapsto \ell_n(\theta)$ on the compact set
$\Theta$ guarantees the existence of at least one maximiser.
To establish consistency we still need a law of large numbers.
Proposition~\ref{lemma:ULLN} below shows that, under
Assumptions~\ref{assumption:compacity_param}--\ref{assump:gc},
the normalised log-likelihood 
converges uniformly to its expectation.
\end{preuve}


\begin{proposition}[Uniform Law of Large Numbers (LLN)]\label{lemma:ULLN}
Under Assumptions~\ref{assumption:compacity_param} to \ref{assump:gc}, 
\[
  \sup_{\theta\in\Theta}\Bigl|\frac{1}{n}\ell_n(\theta)-\E\left[\ell_n(\theta)\right]\Bigr|\xrightarrow{a.s.}0.
\]
where the expectation is defined wrt the distribution of  $T$.
\end{proposition}


\begin{preuve}[Proposition \ref{lemma:ULLN}]
From Proposition~\ref{prop:C2-networks}, the likelihood $\ell_n(\theta)$ is differentiable on the compact set $\Theta$. With $\log \ell_n(\theta) =  \sum_{i=1}^N\sum_{j=1}^{n(i)} \tilde{\ell}^{(i)}_j(\theta)$ 
with
\[
\tilde{\ell}^{(i)}_j(\theta)
    = \delta^{(i)}_j\left(\log\beta-\log\eta
     +(\beta-1)\bigl[\log z^{(i)}_j-\log\eta\bigr]\right)
     -(z^{(i)}_j/\eta)^{\beta},
\]
(simplifying temporarily the notations $\beta=\beta(\theta,\bX),\;
     \eta=\eta(\theta,\bX)$), then
\begin{eqnarray}
  \nabla_\theta\log \ell_n(\theta)
  & =& 
  \frac{\partial\log \ell_n(\theta)}{\partial\beta}\,
        \nabla_\theta\beta(\theta,X)
  \;+\;
  \frac{\partial\log \ell_n(\theta)}{\partial\eta}\,
        \nabla_\theta\eta(\theta,X)\label{eq:grad_loglik}
\end{eqnarray}
with
\begin{eqnarray*}
\frac{\partial\log \ell_n(\theta)}{\partial\beta} & = & \sum\limits_{i=1}^N\sum\limits_{j=1}^{n(i)} \delta^{(i)}_j\left(\frac{1}{\beta} + \log(z^{(i)}_j/\eta) -(z^{(i)}_j/\eta)^{\beta}\log (z^{(i)}_j/\eta)\right), \\
\frac{\partial\log \ell_n(\theta)}{\partial\eta} & = & \sum\limits_{i=1}^N\sum\limits_{j=1}^{n(i)} \frac{\beta}{\eta} \left({(z^{(i)}_j/\eta)^{\beta}}-1\right). 
\end{eqnarray*}
Under bounding constraints (\ref{constraint-eta-1}) and (\ref{constraint-beta-1}), and defining $s^{(i)}_j=(z^{(i)}_j/\eta)^{\beta}$, one gets the bounds
\begin{eqnarray*}
\left|\frac{\partial\log \ell_n(\theta)}{\partial\beta}\right| & \leq  &  C_1\sum\limits_{i=1}^N\sum\limits_{j=1}^{n(i)} \delta^{(i)}_j\left(1 + |\log s^{(i)}_j| +  s^{(i)}_j |\log s^{(i)}_j|\right), \\
\left|\frac{\partial\log \ell_n(\theta)}{\partial\eta} \right| & \leq & C_2 \sum\limits_{i=1}^N\sum\limits_{j=1}^{n(i)}\left(1+s^{(i)}_j\right)
\end{eqnarray*}
with $C_1=\beta^{-1}_{\min}$ and $C_2=\beta_{\max}/\eta_{\min}$.
Since $s^{(i)}_j$ given $\delta^{(i)}_j=1$ is a realization of the random variable $(T/\eta)^{\beta}|\theta,\bX\sim {\cal{E}}(1)$, then $\E[|\log S|]<\infty$ and $\E[S|\log S|]<\infty$. Furthermore, when $\delta^{(i)}_j=0$, as the time allotted to any mission is finite, we always have $\E[S^{(i)}_j]<\infty$. Hence the upper bounds above are integrable. Hence, from  Proposition \ref{prop:C2-networks}, it exists $\Delta_{1}>0$  such that the score (\ref{eq:grad_loglik}) is such that
\begin{eqnarray}
\sup\limits_{\theta\in\Theta} \| \nabla_{\theta} \log\ell_n(\theta)\| & \leq & \Delta_1 R^L \left[n_u F_u(\bZ) + n_c F_c(\bZ)\right]
\label{eq:lipschitz_1}
\end{eqnarray}
where $(F_u(\bZ),F_c(\bZ))$ are $L^1$-integrable functions of the whole sample $\bZ$. 
Besides, for any $(\theta,{\theta}')\in\Theta$, applying the mean-value integral formula to  
$g(u)=\log\ell_n\!\bigl(\theta+u(\theta'-\theta)\bigr)$ gives
\begin{eqnarray}
   \log\ell_n(\theta')-\log\ell_n(\theta)
    & = & \int_{0}^{1}
       \nabla_{\theta}\log\ell_n
       \!\bigl(\theta+u(\theta'-\theta)\bigr)\,du
     \;(\theta'-\theta).\label{eq:integral}
\end{eqnarray}
Hence from (\ref{eq:lipschitz_1}) we get a random but $L^1$- integrale Lipschitz modulus for the log-likelihood within $\Theta$, whose the expectancy depends on the size of censored and uncensored observations and the depth of the network: denoting $\bar{F}_u=\E[F_u(\bZ)]$ and $\bar{F}_c=\E[F_c(\bZ)]$, then
\begin{eqnarray*}
   \left|\frac{1}{n}\log\ell_n(\theta')-\frac{1}{n}\log\ell_n(\theta)\right|
    & \leq  & \Delta_1 R^L\left(\rho_{1,n}\bar{F}_u + \rho_{2,n}\bar{F}_c\right)\|\theta'-\theta\|.\label{eq:lipschitz_2}
\end{eqnarray*}
where $\rho_{1,n}={n_u}/{n}$ and $\rho_{2,n}={n_c}/{n}$. From Assumption (\ref{assum:censoring_vanish}), which implies that $n_u=O(n)$, these two terms are upperly bounded. Remind that $\Theta$ is defined as the ball ${\cal{B}}_{p_n}(0,R)$ where  $p_n = \mbox{Card}(\theta)$. Hence, for any $\varepsilon>0$, the bracketing number $N_{\cal{B}}(\varepsilon,\mathcal {G}_n,L^1(\Pp))$ for the class of bounded functions ${\mathcal{G}}_n  =  \{ \frac{1}{n}\log\ell(\theta): \ \theta\in\Theta\}$ is such that
\begin{eqnarray}
\log N_{\cal{B}}(\varepsilon,\mathcal {G}_n,L^1(\Pp)) & \leq & p_n \log \left( \frac{\mu}{\epsilon}\right).\label{eq:bound_logN}
\end{eqnarray}
with $\mu=4 \left(\rho_{1,n}\bar{F}_u + \rho_{2,n}\bar{F}_c\right)\Delta_1 R^{L+1}\sqrt{p_n}$.
Hence, the Levin-Chandra entropy integral is
\begin{eqnarray*}
\int_0^{1} \sqrt{N_{\cal{B}}(u,\mathcal {G}_n,L^1(\Pp))}\ du & \leq & {\sqrt{{p_n}}} \int_0^{1} \sqrt{\log \frac{\mu}{u} } \ du, \\
& \leq & {\sqrt{{p_n}}} \left(\sqrt{\log \mu +1} + \mbox{A}\right)
\end{eqnarray*}
where $A$ is a constant. Note that 
\begin{eqnarray}
p_n & = & \underbrace{m_1(d+1)}_{\text{\tiny first layer}} + \sum\limits_{\ell=2}^{l-1}\underbrace{m_{\ell}(m_{\ell-1}+1)}_{\text{\tiny $\ell$ layer}} + \underbrace{2m_L}_{\text{\tiny two outputs}}. \label{eq:expr_pn}
\end{eqnarray}
Then $p_n\leq m_{\max}(m_{\max}+1)L + \text{const}$. Hence $p_n={\cal{O}}(m^2_{\max}L)$.
Finally, $\mu\sim R^L\sqrt{p_n}$ and 
\begin{eqnarray*}
\int_0^{1} \sqrt{N_{\cal{B}}(u,\mathcal {G}_n,L^1(\Pp))}\ du & = & {\cal{O}}\left(\sqrt{p_n \log \mu}\right), \\
& = & {\cal{O}}\left(m_{\max} L\right), \\
& = & {o}\left(\sqrt{n}\right) \ \ \ \text{under Assumption (\ref{assump:gc}).}
\end{eqnarray*}
From Theorem 19.4 in \cite{van1996weak}, the class ${\mathcal{G}}_n$ is Glivenko-Cantelli, and consistency then follows from Theorem~5.7 of \cite{van2000asymptotic}. $\square$
\end{preuve}

\begin{preuve}[Proposition \ref{prop:fisher}]
Using the simplified notations $\beta = \beta(\theta, \bX)$, $\eta = \eta(\theta, \bX)$, and $s^{(i)}_j = (z^{(i)}_j / \eta)^{\beta}$ as introduced in the proof of Proposition~\ref{lemma:ULLN}, we define
\begin{eqnarray*}
V^{(i)}_{1,j} & = & \delta^{(i)}_j(1/\beta + \log z^{(i)}_j/\eta) - s^{(i)}_j \log z^{(i)}_j/\eta, \\[3mm]
V^{(i)}_{2,j} & = & \frac{\beta}{\eta} (s^{(i)}_j-\delta^{(i)}_j),
\end{eqnarray*}
and denote $V_1=\sum_{i=1}^N \sum_{j=1}^{n(i)} V^{(i)}_{1,j}$ and $V_2=\sum_{i=1}^N \sum_{j=1}^{n(i)} V^{(i)}_{2,j}$ then $\bV=(V_1,V_2)$. 
From (\ref{eq:grad_loglik}), the score can be rewritten as $\nabla_{\theta} \log \ell_n(\theta) = \bV J_{\theta}(\bX)$,  
where $J_{\theta}(\bX) = \left( \nabla_{\theta} \beta(\theta, \bX),\; \nabla_{\theta} \eta(\theta, \bX) \right)^T$.  
The Fisher information from the $n$ observations can then be written as
\begin{eqnarray}
I_n(\theta) & = & \E\left[\nabla_{\theta}\log \ell_n(\theta)\nabla_{\theta}\log \ell_n(\theta)\right], \nonumber \\
& = & \E\left[V^2_1\nabla_{\theta} \beta \nabla_{\theta} \beta^T + V^2_2\nabla_{\theta} \eta \nabla_{\theta} \eta^T + 2 V_1V_2 \nabla_{\theta} \beta\nabla_{\theta} \eta^T\right], \nonumber \\
& = & \E\left[J_{\theta}(\bX)^T \bS(\bX) J_{\theta}(\bX)\right] \label{eq:carac_Fisher}
\end{eqnarray}
with $\bS(\bX)=\bV^T\bV\in \R^{2\times 2}$.   
As $\beta$ and $\eta$ are bounded, using the same arguments on integrability than in the proof of Proposition \ref{lemma:ULLN}, $\E[V^2_1]>0$ and $\E[V^2_2]>0$ and $I_n(\theta)$ is well-defined and always finite. Furthermore $\mbox{Var}[V_1|\bX]>0$ and $\mbox{Var}[V_2|\bX]>0$, hence $\bS(\bX)$ is positive definite and all entries of $\bS(\bX)$ are bounded and separated from 0 uniformly in $\bX$. Therefore there exists a constant $c_0>0$, depending only on bounds on $(\beta,\eta)$, such that $\bS(\bX)\succeq c_0 I_2$ for every $\bX$.
Then, from (\ref{eq:carac_Fisher}), for any non-zero $\upsilon\in\R^q$,
\begin{eqnarray*}
\upsilon^T I_n(\theta) \upsilon & = & \E\left[(J_{\theta}(\bX)\upsilon)^T \bS(\bX) (J_{\theta}(\bX)\upsilon)\right], \\
& \geq & c_0 \E\| J_{\theta}(\bX)\|^2 \\
& = & c_0 \upsilon^T\Sigma(\theta)\upsilon
\end{eqnarray*}
where $\Sigma(\theta)  =  \E[J_{\theta}(\bX)^TJ_{\theta}(\bX)] \in\mathbb{R}^{q\times q}$, ie. 
\begin{eqnarray*}
\Sigma(\theta)  
& = & \left(\begin{array}{cc}
    \mbox{Var}\!\bigl[\nabla_{\theta}\beta(\theta,X)\bigr]
    &
    \mbox{Cov}\!\bigl[\nabla_{\theta}\beta(\theta,X),\,
                              \nabla_{\theta}\eta(\theta,X)\bigr] \\[6pt]
\mbox{Cov}\!\bigl[\nabla_{\theta}\eta(\theta,X),\,
                              \nabla_{\theta}\beta(\theta,X)\bigr]
    &
    \mbox{Var}\!\bigl[\nabla_{\theta}\eta(\theta,X)\bigr]
  \end{array}\right).
\end{eqnarray*}
From Lemma \ref{lemma:fisher_ok} it comes $\Sigma(\theta)\succ 0$. Then from (\ref{eq:carac_Fisher}), $I_n(\theta)\succ 0$.$\square$
\end{preuve}

\begin{lemma}\label{lemma:fisher_ok}
Under Assumption \ref{assump:distinct_heads}, for any non-zero direction $\upsilon\in \R^q$, 
\begin{eqnarray*}
\mbox{Var}[\upsilon^T\nabla_{\theta}\beta(\theta,\bX)] + \mbox{Var}[\upsilon^T\nabla_{\theta} \eta(\theta,\bX)] & > & 0.
\end{eqnarray*}
\end{lemma}

\begin{preuve}[Lemma \ref{lemma:fisher_ok}]
Because \(\beta\) and \(\eta\) are obtained by linear heads on the shared
representation \(h_{L-1}(X)\),
the chain rule yields matrices
\(B_\beta,B_\eta\in\mathbb R^{m_{L-1}\times q}\) such that
\[
\nabla_\theta\beta(\theta,X)=B_\beta^{\top}h_{L-1}(X)
\ \ \text{and}
\ \ \ 
\nabla_\theta\eta(\theta,X)=B_\eta^{\top}h_{L-1}(X).
\]
Define \(a(v):=B_\beta v\) and \(b(v):=B_\eta v\) in \(\mathbb R^{m_{L-1}}\).
Then
\[
v^{\top}\nabla_\theta\beta(\theta,X)=a(v)^{\top}h_{L-1}(X) \ \ \text{and}
\ \ \ 
v^{\top}\nabla_\theta\eta(\theta,X)=b(v)^{\top}h_{L-1}(X).
\]
From (\ref{eq:A7-2}) (Assumption \ref{assump:distinct_heads}),
\(\operatorname{Var}[u^{\top}h_{L-1}(X)]>0\) for every
\(u\neq0\).
Consequently
\[
\operatorname{Var}\!\bigl[v^{\top}\nabla_\theta\beta\bigr]=0
\;\Longleftrightarrow\;a(v)=0 \ \ \text{and}
\ \ \ 
\operatorname{Var}\!\bigl[v^{\top}\nabla_\theta\eta\bigr]=0
\;\Longleftrightarrow\;b(v)=0 .
\]
Assume, for contradiction, that both variances vanish. Then \(a(v)=b(v)=0\).  In particular the first rows of
\(B_\beta\) and \(B_\eta\)—namely the two head vectors
\(w_\beta\) and \(w_\eta\)—satisfy
\((w_\beta-w_\eta)^{\top}h_{L-1}(X)=0\) almost surely.  Because \(h_{L-1}(X)\) has full covariance and
$W_\beta$ and $W_\eta$ are not collinear  (Assumption~\ref{assump:distinct_heads}),
this is impossible unless \(v=0\), contradicting our choice of \(v\). Therefore at least one of the two variances is positive, and their sum is
strictly positive.
\end{preuve}

\begin{preuve}[Theorem \ref{theo:clt}]
We adopt the shorthand notation \( S_n(\theta) = \nabla_{\theta} \log \ell_n(\theta) \) and recall the notation $p_n=\mbox{Card}(\theta)$. Denote  $\psi^{(i)}_j(\theta)=\nabla_{\theta} \log f^{(i)}_j(\theta|{\cal{F}}^{(i)}_{j-1})$ the individual score for observation $Z^{(i)}_j$. Using the first-order condition \( S_n(\hat{\theta}_n) = 0 \), a Taylor expansion of $S_n(\theta)$ 
between \(\theta_0\) and \(\hat{\theta}_n\) yields
\begin{eqnarray}
0 & = & S_n(\theta_0) + \nabla^2_{\theta\theta} \log \ell_n(\theta_0) h(1+o(1)) \label{eq:taylor}
\end{eqnarray}
where \( h = \hat{\theta}_n - \theta_0 \). 
From Lemma~\ref{lemma:lindererg} beneath, the individual score terms
$\psi_j^{(i)}(\theta_0)$
satisfy a uniform \( L^{2+\delta'} \)-bound and a Lindeberg-type condition. Hence the central limit theorem applies:
\[
\frac{S_n(\theta_0)}{\sqrt{n}} \xrightarrow{d} \mathcal{N}(0, I(\theta_0)).
\]
These score terms are iid, centered, with covariance $I(\theta_0)$ which exists and is not degenerate. From the above result, it comes $\E\|S_n(\theta_0)\|^2 = n \mbox{tr}(I(\theta_0))={\cal{O}}(n p_n)$, then $\|S_n(\theta_0)\|={\cal{O}}(\sqrt{n p_n})$. Now, from Proposition \ref{proposition:conv_hessian} beneath, 
\begin{eqnarray*}
-\frac{1}{n}\nabla^2_{\theta\theta} \log \ell_n{\theta} & \xrightarrow[]{\Pp} &  I(\theta_0).
\end{eqnarray*}
Hence the inverse of the left-hand term remains bounded in probability, and
\begin{eqnarray}
\|h\| & = & {\cal{O}}\left(\left[-\frac{1}{n}\nabla^2_{\theta\theta} \log \ell_n{\theta}\right]^{-1} \frac{\|S_n(\theta_0)\|}{n}\right) \ = \  {\cal{O}}\left(\sqrt{p_n/n}\right). \label{eq:norm_h}
\end{eqnarray}
Consider now the integral remainder in (\ref{eq:taylor})
\begin{eqnarray}
R_n(\theta_0) = \left[\int_0^1 \left[ \nabla^2_{\theta\theta} \log \ell_n(\theta_0 + sh) - \nabla^2_{\theta\theta} \log \ell_n(\theta_0) \right] ds \right] h. \label{eq:hessian_difference}
\end{eqnarray}
such that
\begin{eqnarray}
-\frac{1}{n} \nabla^2_{\theta\theta} \log \ell_n(\theta_0) \sqrt{n} h &=& \frac{S_n(\theta_0)}{\sqrt{n}} + \frac{R_n(\theta_0)}{\sqrt{n}}.\label{eq:equilibrium}
\end{eqnarray}
Hence
\begin{eqnarray}
\frac{R_n(\theta_0)}{\sqrt{n}} & \leq & \sup\limits_{\|\theta-\theta_0\|\leq \|h\|}\left\|-\frac{1}{n}\nabla^2_{\theta\theta} \log \ell_n{\theta} - I(\theta_0)\right\|\sqrt{n}\| h \|. \nonumber
\end{eqnarray}
From (\ref{eq:norm_h}), one has $\sqrt{n}\|h\|={\cal{O}}(\sqrt{p_n})$. Besides, from Proposition \ref{proposition:conv_hessian},  $$\sup\limits_{\|\theta-\theta_0\|\leq \|h\|}\left\|-\frac{1}{n}\nabla^2_{\theta\theta} \log \ell_n{\theta} - I(\theta_0)\right\|=o(1).$$ Then ${R_n(\theta_0)}/{\sqrt{n}}=o(1)$ if $\sqrt{p_n}o(1)\to 0$, which is true for $p_n=o(n^{1/3})$ (Assumption \ref{assump:gc}). Finally, from \eqref{eq:equilibrium} and the convergence 
$
\frac{1}{n} \nabla^2_{\theta\theta} \log \ell_n(\theta) \xrightarrow{P} -I(\theta)$, we conclude that
\[
\sqrt{n}h \ = \ \sqrt{n}(\hat\theta_n - \theta) \xrightarrow{d} \mathcal{N}(0, I(\theta)^{-1}). \ \square
\]
\end{preuve}

\begin{proposition}[Uniform LLN for the Hessian]\label{proposition:conv_hessian}
Denote $I(\theta_0)=
  \E_{\theta_0}\!\Bigl[\,\nabla_{\!\theta}\ell_1(\theta_0)\;
      \nabla_{\!\theta}\ell_1(\theta_0)^{\!\top}\Bigr]$ the unit Fisher information matrix at $\theta_0$.
      Under Assumptions \ref{assumption:compacity_param} to \ref{assump:distinct_heads},
there exists a decreasing sequence $r_n\to 0$, such that $r_n\sqrt{n}\to \infty$ and 
\begin{eqnarray*}
\sup\limits_{\|\theta-\theta_0\|\leq r_n} \left\|-\frac{1}{n}\nabla^2_{\theta\theta} \log \ell_n{\theta} - I(\theta_0)\right\| & \xrightarrow[]{\Pp} & 0.
\end{eqnarray*}
\end{proposition}

\begin{preuve}[Proposition \ref{proposition:conv_hessian}]
Using the same notations than in the proof of Proposition \ref{lemma:ULLN},  the Hessian can be written as
\begin{eqnarray*}
 \nabla^2_{\theta\theta} \log \ell_n(\theta) & = & \sum\limits_{i=1}^N\sum\limits_{j=1}^{n(i)} \delta^{(i)}_j H^{(i)}_{j,u}(\theta) + (1-\delta^{(i)}_j) H^{(i)}_{j,c}(\theta)
\end{eqnarray*}
where
\begin{eqnarray*} 
H^{(i)}_{j,u}(\theta) & = & -\bigl(\log\frac{Z^{(i)}_j}{\eta}-s^{(i)}_{j}\bigr)\,
  \nabla_{\!\theta}^2\beta
-\frac{\beta}{\eta}\,(s^{(i)}_{j}-1)\,
  \nabla_{\!\theta}^2\eta
+\Bigl(-\tfrac{1}{\beta}+s^{(i)}_{j}\Bigr)\,
  \bigl(\nabla_{\!\theta}\beta\bigr)
  \bigl(\nabla_{\!\theta}\beta\bigr)^{\!\top} \\
&& +\frac{\beta}{\eta^{2}}\,(1-s^{(i)}_{j})\,
  \bigl(\nabla_{\!\theta}\eta\bigr)
  \bigl(\nabla_{\!\theta}\eta\bigr)^{\!\top} 
+\frac{s^{(i)}_{j}-1}{\eta}\,
  \Bigl[
    \bigl(\nabla_{\!\theta}\beta\bigr)
    \bigl(\nabla_{\!\theta}\eta\bigr)^{\!\top}
   +\bigl(\nabla_{\!\theta}\eta\bigr)
    \bigl(\nabla_{\!\theta}\beta\bigr)^{\!\top}
  \Bigr], \\   
  H^{(i)}_{j,c}(\theta) & = &-\,s^{(i)}_{j}\,\log_{ij}\,
  \nabla_{\!\theta}^2\beta
+\;s^{(i)}_{j}\,
  \frac{\beta}{\eta}\,
  \nabla_{\!\theta}^2\eta 
+\;s^{(i)}_{j}\,\log_{ij}^{\,2}\,
  \bigl(\nabla_{\!\theta}\beta\bigr)
  \bigl(\nabla_{\!\theta}\beta\bigr)^{\!\top} \\
&&-\;s^{(i)}_{j}\,
  \frac{\beta}{\eta^{2}}\,
  \bigl(\nabla_{\!\theta}\eta\bigr)
  \bigl(\nabla_{\!\theta}\eta\bigr)^{\!\top} 
-\;s^{(i)}_{j}\,
  \frac{\log_{ij}}{\eta}\,
  \Bigl[
    \bigl(\nabla_{\!\theta}\beta\bigr)
    \bigl(\nabla_{\!\theta}\eta\bigr)^{\!\top}
   +\bigl(\nabla_{\!\theta}\eta\bigr)
    \bigl(\nabla_{\!\theta}\beta\bigr)^{\!\top}
  \Bigr]. 
\end{eqnarray*}
The existence of $\E\left[\nabla_{\theta}\log \ell_n(\theta)\nabla_{\theta}\log \ell_n(\theta)\right]$ (Proposition~\ref{prop:fisher}) implies that \(\mathbb{E}[\nabla^2_{\theta\theta} \log \ell_n(\theta)] < \infty\). This result can also be established by noting that the Hessian \(\nabla^2_{\theta\theta} \log \ell_n(\theta)\) can be written as a sum of matrices involving first- and second-order derivatives of the mappings \(\theta \mapsto \beta(\theta, \mathbf{X})\) and \(\theta \mapsto \eta(\theta, \mathbf{X})\), each multiplied by scalar functions of \((Z, \mathbf{X})\) whose expectations are finite due to the moment condition \(\mathbb{E}[Z^m] < \infty\) for all \(m > 0\), a property satisfied by the Weibull distribution.
From Proposition \ref{prop:C2-networks}, there exist $(M_1,M_2)>0$ such that
\begin{eqnarray}
  \sup_{\theta,X}\bigl\lVert\partial_{\bm\theta}\beta\bigr\rVert\le M_1,
  \quad
  \sup_{\theta,X}\bigl\lVert\partial_{\bm\theta}^2\beta\bigr\rVert
  +\sup_{\theta,X}\bigl\lVert\partial_{\bm\theta}^2\eta\bigr\rVert\le M_2. \label{eq:upper-bounds}
\end{eqnarray}
With $\lvert\log Z-\log\eta(\theta)\rvert
         \le |\log Z|+|\log\eta_{\min}^{-1}|
         \le C_0(1+T)$ since $\eta_{\min}\le\eta(\theta)\le\eta_{\max}$, $|\log Z|\le 1+Z$ and $0\le s(Y,\theta)\le
        T^{\beta_{\max}}/\eta_{\min}^{\beta_{\max}}$ with
        $\beta_{\max}:=\sup_{\theta}\beta(\theta)$ hence
       $|1-s|\le 1+s$, and based on the same (shortened) arguments than the proof of Proposition (\ref{lemma:ULLN}), it exists $C>0$ such that  
\[
  \sup_{\theta\in\Theta}\lVert \nabla^2_{\theta\theta} \log \ell_n(\theta)\rVert
  \;\le\;
  n C g_n(\bZ),
\]
with 
\begin{eqnarray*}
g_n(\bZ) & =& \frac{1}{n} \sum\limits_{i=1}^N\sum\limits_{j=1}^{n(i)} \bigl(1+|\log z^{(i)}_j|+(z^{(i)}_j)^{\beta_{\max}}\bigr)
\end{eqnarray*}
being finite and square integrable (from the properties of the Weibull distribution). Then $n C g_n(\bZ)$ is a $L^2$ uniform envelope for $\nabla^2_{\theta\theta} \log \ell_n(\theta)$. 
The Glivenko-Cantelli theorem states that
\begin{eqnarray*}
\sup\limits_{\theta\in\Theta} \left\|\frac{1}{n}\nabla^2_{\theta\theta} \log \ell_n(\theta) - \E_{\theta_0}\left[\nabla^2_{\theta\theta}\log \ell_1(\theta)\right]\right\| & \xrightarrow[]{\Pp} & 0.
\end{eqnarray*}
Then, by reducing $\Theta$ to the ball ${\cal{B}}(\theta_0,r_n)$ with $r_n\to 0$ and $r_n\sqrt{n}\to\infty$ (for instance $r_n=n^{-1/4}$), since $I(\theta)$ is continuous, it comes
\begin{eqnarray*}
\sup\limits_{\|\theta-\theta_0\|\leq r_n} \left\|-\frac{1}{n}\nabla^2_{\theta\theta} \log \ell_n{\theta} - I(\theta_0)\right\| & \xrightarrow[]{\Pp} & 0. \  \square
\end{eqnarray*}

\end{preuve}

\begin{lemma}[Moment - Lindeberg]\label{lemma:lindererg}
For all $i\in\{1,\ldots,N\}$ and $j\in\{1,\ldots,n(i)\}$, denote $\psi^{(i)}_j(\theta)=\nabla_{\theta} \log f^{(i)}_j(\theta|{\cal{F}}^{(i)}_{j-1})$ the individual score for observation $Z^{(i)}_j$. Under Assumptions \ref{assumption:compacity_param} to \ref{assump:distinct_heads}, there exists $\delta'>0$ such that, for all $\varepsilon>0$
\begin{eqnarray}
\sup\limits_{\theta\in\Theta} \E\left\| \psi^{(i)}_j(\theta)\right\|^{2+\delta'} & < & \infty.\label{eq:lindeberg-1}
\end{eqnarray}
Besides, 
\begin{eqnarray}
 \frac{1}{n}\sum\limits_{i,j} \E\left\| \psi^{(i)}_j(\theta_0)\right\|^2 \1_{\{\|\psi^{(i)}_j(\theta_0)\|\geq \varepsilon\sqrt{n}\}}& \xrightarrow[n\to\infty]{} & 0 \label{eq:lindeberg-2}
\end{eqnarray}
\end{lemma}

\begin{preuve}[Lemma \ref{lemma:lindererg}]
{\bf Inequation (\ref{eq:lindeberg-1}).} One can prove (\ref{eq:lindeberg-1}) based on the proofs and notations of Proposition \ref{prop:C2-networks} and Proposition \ref{lemma:ULLN}. With
\begin{eqnarray}
\psi^{(i)}_j(\theta) & = & 
  A^{(i)}_j(\theta)\;\nabla_{\!\theta}\beta
  \;+\;
  \frac{\beta}{\eta}\Bigl(s^{(i)}_j-\delta^{(i)}_j\Bigr)\;
  \nabla_{\!\theta}\eta \label{eq:score_indiv}
\end{eqnarray}
with $A^{(i)}_j(\theta):=
  \delta^{(i)}_j\bigl(\log({Z^{(i)}_j}/{\eta})-s^{(i)}_j\bigr)
  -\bigl(1-\delta^{(i)}_j\bigr)\,s^{(i)}_j(\theta)\,\log({Z^{(i)}_j}/{\eta})$, 
then there exists $M_1>0$ such that 
\begin{eqnarray*}
\sup_{\theta\in\Theta}\left\| \psi^{(i)}_j(\theta)\right\| & \leq & M_1\,|A^{(i)}_j(\theta)|
        +M_1\,\frac{\beta_{\max}}{\eta_{\min}}\,
             \bigl|s^{(i)}_j-\delta^{(i)}_j\bigr|
\end{eqnarray*}
where $M_1:=\sup_{\theta\in\Theta}
      \max\{\|\nabla_{\!\theta}\beta\|,
            \|\nabla_{\!\theta}\eta\|\}$.
With $|A^{(i)}_j(\theta)|\le
  |\log({Z^{(i)}_j}/{\eta})|
  +s^{(i)}_j(1+|\log({Z^{(i)}_j}/{\eta})|)$,  $s^{(i)}_j-\delta^{(i)}_j\bigr|\le
  1+s^{(i)}_j$ and
\[
s^{(i)}_j(\theta)=
  \bigl(Z^{(i)}_j/\eta\bigr)^{\beta}
  \;\le\;
  \bigl(Z^{(i)}_j/\eta_{\min}\bigr)^{\beta_{\max}}
\quad \text{and} \quad
|\log({Z^{(i)}_j}/{\eta})|\le 1+Z^{(i)}_j,
\]
there exists $C>0$ such that 
\begin{eqnarray}
\sup_{\theta\in\Theta}
   \|\psi^{(i)}_j(\theta)\| & \le &  C\Bigl(
      1+|\log Z^{(i)}_j|
      +(Z^{(i)}_j)^{\beta_{\max}}
      +(Z^{(i)}_j)^{\beta_{\max}}\!\,|\log Z^{(i)}_j| \label{eq:control_score_individual}
     \Bigr)
\end{eqnarray}
With a Weibull $Z^{(i)}_j$, \(\E[(Z^{(i)}_j)^{m}]<\infty\) for all $m>0$, and in particular for $m=2+\delta'$. Hence the right term of (\ref{eq:control_score_individual}) is a $L^{2+\delta'}$ uniform envelop for the individual score $\psi^{(i)}_j(\theta)$. \\

\noindent {\bf Convergence (\ref{eq:lindeberg-2}).} From (\ref{eq:control_score_individual}) and using Markov's inequality,
\[
\mathbf 1_{\{\lVert\psi^{(i)}_j(\theta_0)\rVert>\sqrt n\,\varepsilon\}}
   \;\le\;
   (\sqrt n\,\varepsilon)^{-\delta'}\,
   \lVert\psi^{(i)}_j(\theta_0)\rVert^{\delta'}
   \;\le\;
   (\sqrt n\,\varepsilon)^{-\delta'}\,\left(\Psi^{(i)}_j\right)^{\,\delta'}.
\]
Then 
\[
\begin{aligned}
 &\frac1n\sum_{i=1}^{N}\sum_{j=1}^{n_i}
  \E\!\Bigl[\lVert\psi^{(i)}_j(\theta_0)\rVert^{2}\,
          \mathbf 1_{\{\lVert\psi^{(i)}_j(\theta_0)\rVert>\sqrt n\,\varepsilon\}}\Bigr]  \\
 &\qquad\le
   (\sqrt n\,\varepsilon)^{-\delta}\,
   \frac1n\sum_{i,j}\E\bigl[\|\Psi^{(i)}_{j}(\theta_0)\|^{\,2+\delta'}\bigr]
   \;=\;
   \frac{C'}{n^{1+\delta'/2}}
   \;\xrightarrow[n\to\infty]{}0 \ \square
\end{aligned}
\]
\end{preuve}

\begin{preuve}[Proposition \ref{prop:archi}]
It is straigthforward to check that when $n>n_{\min}$ and $K\geq K_0$, then $m_1\geq d$. 
From (\ref{eq:expr_pn}), the total number of parameters can be written
\begin{eqnarray*}
p_n 
& = & m_1(d+1)
+ \sum_{\ell=2}^{L-1} m_\ell\,(m_{\ell-1}+1)
+ 2m_L.
\end{eqnarray*}
Hence $p_n\le m_{\max}(m_{\max}+1)L + \mathrm{const}\le 2\,m_{\max}^2 L + \mathrm{const}$ for $m_{\max}\ge 1$.
Using $m_{\max}^2 \le \tfrac{K}{2}\,\dfrac{n^{1/3}}{\log n}$ and $L_n \le \lceil(\log n)^{\tau}\rceil$, we obtain
\[
p_n \;\le\; 2\cdot \frac{K}{2}\,\frac{n^{1/3}}{\log n}\cdot (\log n)^{\tau} + \mathrm{const}
\;=\; n\,\alpha_n(K,\tau) + \mathrm{const}
\]
with \[
\alpha_n(K,\tau)\;=\; \frac{K\,n^{-2/3}}{(\log n)^{\,1-\tau}}.
\]
Therefore, for large $n$,
\[
p_n \;\le\; K\,\frac{n^{1/3}}{(\log n)^{1-\tau}} \;=\; n\,\alpha_n(K,\tau)
\]
Note that $L=(\log n)^{\tau}=o(n^{2/3})$ and that
\[
p_n \;=\; O\!\Big(\frac{n^{1/3}}{(\log n)^{1-\tau}}\Big) \;=\; o(n^{1/3}).
\]
By construction, for any $\ell\ge 2$ with $m_{\ell-1}\ge 3$ we have
$m_\ell \le m_{\ell-1}-1$, so successive widths are strictly decreasing and the depth is truncated as soon as the width $2$ is reached. For large $n$ the cap $1+(m_{\max}-2)$ dominates
$\lceil(\log n)^{\tau}\rceil$ (since $m_{\max}\asymp n^{1/6}/\sqrt{\log n}$),
so $L_n=\lceil(\log n)^{\tau}\rceil$ and the rates above are unaffected.
\end{preuve}

\section{Supplementary Material}

\subsection{Main features of DPTWE and WTTE-RNN models}

We summarized here the main characteristics of the competing models in comparison with the WTNN approach developed in the main paper. Note that DPWTE must be adapted in practice by many features for a fair comparison. Especially, it produces predictions of Weibull parameters $(\eta,\beta)$ normalized around 1 (without accounting for the well-known link between $\beta$ and the nature of aging \cite{bousquet2010elicitation}). In contrast,  WTTE-RNN can predict the waiting time until the next event, treating time as discrete (seconds, hours, etc.). It estimates at each time step the waiting time to the next event, and allows for bounded intervals for $\beta$. It also offers a method for initializing weights and biases, a key factor for enabling convergence during training. 


\subsubsection{The DPWTE model}

The Deep Parsimonious Weibull Time-to-Event (DPWTE) model 
\cite{bennis2021dpwte,DPWTE} extends the mixture-based strategy for survival analysis 
by addressing a key limitation of earlier methods such as DeepWeiSurv, described hereinafter. Reusing the notations of the main paper, let $T$ denote the event time, $Y$ the censoring time, and $(Z,\delta)$ the observed data with 
$Z=\min(T,Y)$ and $\delta=\mathbf 1\{T\le Y\}$. Let $\mathbf X$ be the covariate vector.  

\paragraph{A first version: DeepWeiSurv.}  
DeepWeiSurv assumes that the conditional distribution of $T$ can be represented as a finite mixture 
of Weibull distributions, with individual-specific parameters depending on $\mathbf X$. Formally, the 
mixture distribution is written as
\[
f(t \mid \mathbf X) \;=\; \sum_{k=1}^p \alpha_k(\mathbf X)\,
f_{W} \!\big(t \mid \beta_k(\mathbf X), \eta_k(\mathbf X)\big),
\]
where $f_W$ denotes the Weibull density with shape $\beta_k$ and scale $\eta_k$, and 
$\alpha_k(\mathbf X)$ are mixture weights such that $\sum_{k=1}^p \alpha_k(\mathbf X)=1$.  
The network architecture consists of a shared feature extractor followed by task-specific heads 
that output $(\alpha_k, \beta_k, \eta_k)$ for $k=1,\dots,p$. Model training is performed via 
the maximization of the log-likelihood accounting for censoring:
\[
\mathcal L = \sum_{i=1}^n \delta_i \log f(Z_i \mid \mathbf X_i)
+ (1-\delta_i)\log S(Z_i \mid \mathbf X_i),
\]
where $S$ is the mixture survival function. This approach provides flexibility and interpretability, 
but the number of mixture components $p$ must be fixed in advance, which may lead to 
underfitting or overfitting in practice.  

\paragraph{The DPWTE version.}  
DPWTE proposes a more flexible and parsimonious solution. Instead of fixing $p$, the method 
sets a sufficiently large upper bound and introduces a \emph{Sparse Weibull Mixture} (SWM) layer. 
This layer selects, through its weights, the subset of Weibull components that significantly contribute 
to the modeling of the distribution of $T$. To encourage sparsity, a dedicated regularization term 
is added to the likelihood-based loss, ensuring that only a small number of Weibull components 
remain active after training. As a result, DPWTE learns both the mixture parameters and its 
effective size in a data-driven manner. 


\subsubsection{The WTTE-RNN model}

The Weibull Time-To-Event Recurrent Neural Network (WTTE-RNN) was originally introduced 
by Martinsson \cite{WTTERNNW} as a framework to jointly exploit survival analysis and 
recurrent neural networks for sequential prediction of time-to-event outcomes. 

Its core construction principle is the following.  
At each time step $t$, the hidden state of an RNN encodes the history of $\mathbf X$ up to $t$, 
and outputs the parameters $(\alpha_t,\beta_t)$ of a Weibull distribution for the conditional time 
to the next event. The loss function is the log-likelihood adapted to censored observations:
\[
\mathcal L = \sum_{i=1}^n \delta_i \log f(Z_i \mid \alpha_i,\beta_i) 
+ (1-\delta_i) \log S(Z_i \mid \alpha_i,\beta_i),
\]
where $f$ and $S$ are the Weibull density and survival functions. This directly incorporates 
censoring into the training objective and allows the model to handle continuous or discrete 
waiting times, recurrent events, and time-varying covariates in a unified way.

Since its introduction, several improvements have been proposed. Cawley and Burns 
\cite{cawley2019analysis} analyzed different neural architectures (LSTM, GRU, Phased LSTM) 
that improve performance on predictive maintenance tasks. Xu et al. \cite{xu2022method} extended 
the model by combining CNN layers with LSTMs (WTTE-CNN-LSTM) to enhance feature 
extraction from raw sensor streams, achieving better accuracy for remaining useful life (RUL) 
prediction. Li et al. \cite{li2022estimation} further adapted the WTTE-RNN to equipment RUL 
estimation, while Dhada et al. \cite{dhada2023weibull} applied Weibull RNNs to histogram data for 
failure prognosis. These studies confirm the versatility of the WTTE-RNN framework across 
domains such as customer churn, predictive maintenance, and reliability engineering.


\subsection{Implementation choices for training}


Training of neural networks can diverge, resulting in very large loss values leading to NaN outputs. This issue is discussed in detail in \cite{WTTERNNW} [Section~4.1.2], where the authors propose specific initialization strategies for weights and biases, which were used in the present work. 

\subsubsection{Optimization task}

The WTNN model was trained by minimizing the likelihood-based loss function augmented with regularization terms enforcing monotonicity, head non-collinearity, and statistical alignment. The Adam optimizer was employed with an initial learning rate of 0.01, combined with a ReduceLROnPlateau scheduler \cite{al2022scheduling} that decreased the learning rate by a factor of 0.1 if no improvement was observed for 10 epochs. The training loop incorporated mechanisms to prevent divergence and adapt optimization online. Initialization used very small weights and relatively large biases (ratio up to 100:1), based on grid search experiments that avoided exploding losses. A multi-start strategy was used to mitigate non-convexity and promote reproducibility, while  the experiments were conducted using a grid search with various values for different hyperparameters. 




\begin{figure}[H]
    \centering
    \includegraphics[width=0.65\linewidth]{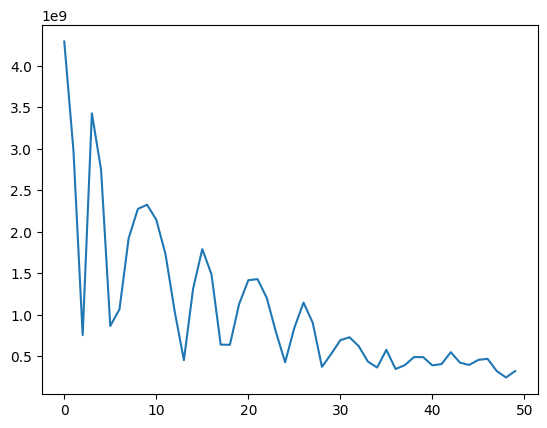}
    \caption{Typical training evolution (training rate vs number of epochs).}
    \label{fig:training evolution}
\end{figure}

\subsubsection{Initialization scale, normalization, and stability}

\paragraph{Main features.}
In this article, we aim to learn smooth, differentiable mappings from covariates to the Weibull parameters $(\beta(x), \eta(x))$, ensuring that the resulting likelihood remains well-behaved for both optimization and calibration. In preliminary experiments, the model under study, a multilayer perceptron with Softplus activations, was trained either with or without Batch Normalization (BN) layers \cite{bjorck2018understanding,santurkar2018does}. See $\S$ \ref{sec:batch_normalization} for details on BN, and further in the text for details on these preliminary experiments.  

Empirically, we observe that the scale of initialization of weights and biases has a decisive effect when BN is absent: smaller initialization magnitudes consistently yield lower losses and better calibration of the Weibull parameters. When BN is inserted after each affine transformation, this sensitivity disappears—performance becomes almost invariant to the initialization scale, even across several orders of magnitude.

This phenomenon can be understood from the properties of the \texttt{Softplus} function $$\phi(x)=\log(1+e^x)$$ whose derivative $\phi'(x)=\sigma(x)$ is the logistic sigmoid. 
\begin{itemize}
\item Without normalization, large weight scales quickly drive pre-activations into regions where $\sigma(x)\approx1$ or $\sigma(x)\approx0$, amplifying or damping gradients as they propagate through layers. 
\item Because \texttt{Softplus} is unbounded, the variance of activations follows a multiplicative growth driven by $\|W^{(\ell)}\|_F^2$ and by the local slope of the activation \cite{glorot2010understanding}. 
\item Small initial weights keep activations near zero, where \texttt{Softplus} behaves almost linearly ($\phi(x)\approx \log 2 + x/2$), thus maintaining a stable gradient flow and smoother optimization surface.
\end{itemize}

Adding BN modifies this regime completely. BN standardizes pre-activations using their batch mean and variance, and then applies a learned affine transform \cite{ioffe2015batch}. This step effectively cancels any scaling in the incoming parameters: if all incoming weights and biases are multiplied by a constant $c$, both the batch mean and variance rescale by $c$, leaving the standardized activations unchanged. As a result, \texttt{Softplus} layers following BN receive approximately invariant input distributions, and the network’s loss landscape becomes largely independent of the initialization scale. Theoretical and empirical studies confirm that BN smooths the loss surface and improves conditioning, leading to more stable gradients and faster convergence \cite{santurkar2018does, klambauer2017self,xiao2018dynamical}. \\

While this scale-invariance is beneficial for MSE objectives, it interacts differently with the Weibull maximum likelihood loss. The log-likelihood couples $\beta$ and $\eta$, making the objective exponentially sensitive to the ratio and correlation of these parameters. 
\begin{itemize}
\item In practice, both $\beta$ and $\eta$ are produced by \texttt{Softplus} transformations of network outputs $(z_\beta,z_\eta)$. 
\item BN modifies the distribution of these logits by applying batch-dependent centering and scaling before Softplus. 
\item Small shifts in $(z_\beta,z_\eta)$ can therefore produce large, nonlinear changes in $(\beta,\eta)$ and consequently in the survival function $S(t|x)=\exp\{-(t/\eta)^\beta\}$. This additional stochasticity can degrade the calibration of survival probabilities, even when the overall likelihood remains low. 
\end{itemize}

Under the MSE objective, such variability is largely benign: the network’s task is to approximate conditional expectations, and BN stabilizes both the forward and backward passes. For the Weibull MLE, however, correct relative scaling between $\beta$ and $\eta$ matters more than invariance to parameter magnitude. Over-normalization may flatten important curvature directions of the likelihood, producing optima that are numerically stable but suboptimal in terms of tail calibration. This explains why, in experiments, small initialization without BN gives the best Weibull likelihoods, while BN yields comparable MSE but slightly worse probabilistic calibration.


\paragraph{Preliminary experiments.}
Multiple experiments were conducted to quantify the impact of the initialization of the weights $\bW$ and biases $\bb$ on model performance, either when the training objective is restricted to the MSE or when it includes the negative log-likelihood. We consider here, as in the simulation study of the article, the sampling mechanism 
\begin{eqnarray*}
   \mathbf{W} \sim {\cal{N}}(\text{a},\text{b}), \\
   \mathbf{b} \sim {\cal{N}}(\text{c},\text{d}),
\end{eqnarray*}
choosing the values within all the possible combinations of $(a,b,c,d)$ in $\{0.001 ; 0.01 ; 0.1 ; 1 ;  10\}$. Training results obtained using the MSE only, without BN,  are provided on Table \ref{tab:simulated_results_RMSE}.

\begin{table}[H]
    \centering
    \caption{Training results from simulated experiments for the BN-free WTNN model, trained with the MSE loss only (averaged over 30 repetitions). }
    \label{tab:simulated_results_RMSE}
    \begin{tabular}{l|lll}
        \hline
        Case & Weight simulation & Bias simulation & RMSE value\\
        \hline
        1 & ${\cal{N}}(0.001, 0.001)$ & ${\cal{N}}(0.01, 0.1)$  & 1.149 \\
        2 & ${\cal{N}}(0.001, 0.001)$  & ${\cal{N}}(1, 0.001)$  &  1.149\\
        3 & $N(0.01, 0.01)$ & ${\cal{N}}(0.01, 1)$ & 1.158 \\
        4 & ${\cal{N}}(0.01, 0.01)$  & ${\cal{N}}(10, 10)$ & 1.501 \\
        5 & ${\cal{N}}(0.01, 0.1)$ & $N(1, 10)$ & 3.65 \\
        6 & ${\cal{N}}(0.1, 0.01)$ & ${\cal{N}}(0.001, 10)$ & 83.5 \\
        7 & $N(0.001, 1)$ & ${\cal{N}}(0.01, 0.01)$ & 648.1 \\
        8 & ${\cal{N}}(1, 0.001)$ & ${\cal{N}}(0.001, 0.01)$ & $6.9 .10^{13}$ \\
        9 & ${\cal{N}}(10, 0.001)$ & ${\cal{N}}(1, 0.001)$ & $ 3.5.10^{23}$\\
        \hline
    \end{tabular}
\end{table}

The results show that even a slight change in the initial weight values has a very strong impact on model performance, including when the data are normalized. These result are strongly stabilized when a BN layer is added after each \texttt{Softplus} layer within the network, as highlighted on Table \ref{tab:simulated_results_RMSE_BN}. 

\begin{table}[H]
    \centering
    \caption{Training results from simulated experiments for the  WTNN model incorporating BN layers, trained with the MSE loss only (averaged over 30 repetitions). }
    \label{tab:simulated_results_RMSE_BN}
    \begin{tabular}{l|lll}
        \hline
        Case & Weight simulation & Bias simulation & RMSE value\\
        \hline
        1 & ${\cal{N}}(0.001, 0.001)$ & ${\cal{N}}(0.01, 0.1)$  & 1.151 \\
        2 & ${\cal{N}}(0.001, 0.001)$  & ${\cal{N}}(1, 0.001)$  &  1.152\\
        5 & ${\cal{N}}(0.01, 0.1)$ & $N(1, 10)$ & 1.496 \\
        8 & ${\cal{N}}(1, 0.001)$ & ${\cal{N}}(0.001, 0.01)$ & 1.150\\
        9 & ${\cal{N}}(10, 0.001)$ & ${\cal{N}}(1, 0.001)$ & 1.150 \\
        \hline
    \end{tabular}
    \label{tab:simulated_results}
\end{table}



\begin{table}[H]
    \centering
    \caption{Best training results with and without BatchNorm (BN) layer (averaged over 30 repetitions). }
    \label{tab:simulated_results_RMSE_BN_2}
    \begin{tabular}{l|lll}
        \hline
        Situation & Weight simulation & Bias simulation & RMSE value\\
        \hline
        With BN layer & ${\cal{N}}(0.001, 0.001)$ & ${\cal{N}}(0.001, 0.001)$  & 1.14878 \\
        without BN layer & ${\cal{N}}(1, 0.1)$  & ${\cal{N}}(0.1, 0.001)$ & 1.14882 \\
        \hline
    \end{tabular}
    \label{tab:best RMSE results}
\end{table}

\vspace{1cm}

The same experiments have been conducted using a training loss incorporating only the Weibull negative log-likelihood, then the hybrid loss componed with the Weibull loss, the MSE and the regularization terms. They are summarized on Tables \ref{tab:simulated_results_Weibull_a} to \ref{tab:simulated_results_Weibull_MSE_BN}. A summary view highlighting best training results is provided in Table \ref{tab:bestresultsweibullallmodel}.   

\begin{table}[H]
    \centering
    \caption{Best training results between instances of the WTNN model. IBS = Integrated Brier Score. CI = C-Index. AUC = Area Under the Curve.}
    \label{tab:bestresultsweibullallmodel}
    \begin{tabular}{l|llllll}
        \hline
         & Weight simulation & Bias simulation & IBS & MSE & CI & AUC\\
Weibull loss only \\
        \cline{1-1}
        & \\
        BN-free WTNN model case 1 & ${\cal{N}}(0.1, 0.01)$ & ${\cal{N}}(0.001, 1) $  & 0.043 & 768 & 0.425 & 0.375\\
        BN-free WTNN model case 2 & ${\cal{N}}(0.01, 0.001)$ & ${\cal{N}}(0.01, 100) $  & 0.044 & 798 & 0.521 & 0.525\\
        BN WTNN model case 1 & ${\cal{N}}(100, 10)$  & ${\cal{N}}(100, 1)$  &  0.053 & 866 & 0.491 & 0.499\\
        BN WTNN model case 2 & ${\cal{N}}(0.1, 0.01)$  & ${\cal{N}}(100, 100)$  &  0.056 & 867 & 0.546 & 0.501\\
        \hline
        Complete loss  & \\
        \cline{1-1}
        & \\
        BN-free WTNN model case 1 & ${\cal{N}}(0.1, 0.01)$ & ${\cal{N}}(1, 0.01) $  & 0.043 & 765 & 0.432 & 0.381\\
        BN-free WTNN model case 2 & ${\cal{N}}(0.01, 0.001)$ & ${\cal{N}}(1, 100) $  & 0.043 & 803 & 0.521 & 0.525\\
        BN WTNN model case 1& ${\cal{N}}(1, 0.001)$  & ${\cal{N}}(0.01, 1)$  &  0.053 & 866 & 0.494 & 0.499\\
        BN WTNN model case 1 & ${\cal{N}}(0.1, 0.01)$  & ${\cal{N}}(0.1, 0.01)$  &  0.055 & 867 & 0.548 & 0.501\\
        \hline
    \end{tabular}
\end{table}

Again, the use of BN leads to close indicators, regardless of the initialization values (Tables \ref{tab:simulated_results_Weibull-BN} and \ref{tab:simulated_results_Weibull_MSE_BN}). Note that the BN-free models based on the Weibull loss only then the complete loss  allow to obtain better performance for IBs and MSE; which are the most important indicator for models choice. However, such a situation can make the training leading to significant instability, which explains the requirement of conducting preliminary experiments, especially those based on simulated datasets.



\subsubsection{Batch Normalization}\label{sec:batch_normalization}

For a hidden layer $\ell$ of the Weibull-Tailored Neural Network (WTNN), let $h^{(\ell-1)} \in \mathbb{R}^{d_{\ell-1}}$ denote the input activations and
\[
z^{(\ell)} = W^{(\ell)} h^{(\ell-1)} + b^{(\ell)}
\]
the corresponding pre-activations before applying the nonlinearity.  
Batch Normalization (BN) operates on each coordinate $j \in \{1,\dots,d_{\ell}\}$ of $z^{(\ell)}$ using the statistics of the current mini-batch
\[
\mu_j^{(\ell)} = \frac{1}{m} \sum_{i=1}^{m} z_{i,j}^{(\ell)}, 
\qquad
(\sigma_j^{(\ell)})^2 = \frac{1}{m} \sum_{i=1}^{m} \!\left(z_{i,j}^{(\ell)} - \mu_j^{(\ell)}\right)^2 ,
\]
where $m$ is the batch size. The normalized activations are then
\[
\hat{z}_{i,j}^{(\ell)} = \frac{z_{i,j}^{(\ell)} - \mu_j^{(\ell)}}{\sqrt{(\sigma_j^{(\ell)})^2 + \varepsilon}} ,
\]
with a small constant $\varepsilon > 0$ for numerical stability.  
BN introduces two learnable affine parameters $(\gamma_j^{(\ell)}, \beta_j^{(\ell)})$ that restore scale and shift:
\[
\tilde{z}_{i,j}^{(\ell)} = \gamma_j^{(\ell)} \hat{z}_{i,j}^{(\ell)} + \beta_j^{(\ell)} .
\]
The transformed activations $\tilde{z}^{(\ell)}$ are then passed through the elementwise \texttt{Softplus} function,
\[
h^{(\ell)} = \phi\!\left(\tilde{z}^{(\ell)}\right)
           = \log\!\left(1 + e^{\,\tilde{z}^{(\ell)}}\right).
\]

At inference time, BN replaces $(\mu_j^{(\ell)}, \sigma_j^{(\ell)})$ by moving averages accumulated during training.  
In the context of WTNN, this normalization re-centers and rescales each layer’s activations independently of the absolute magnitude of $(W^{(\ell)}, b^{(\ell)})$, which reduces the sensitivity of the network to initialization scale and stabilizes the differentiation of the Weibull functional parameters $(\beta(X), \eta(X))$.

\subsection{Predictive uncertainty using the Monte Carlo Dropout (MCD)}

Once the WTNN instances had been trained and validated, their operational deployment necessitated that any inference produced by these models be systematically accompanied by an associated quantification of predictive uncertainty. While the literature has recently witnessed significant advances in this domain, particularly through the emergence of model-agnostic approaches such as conformal prediction \cite{fontana2023conformal}, the direct control afforded over both the network architecture and the training procedure also renders feasible the implementation of more traditional techniques, among which Monte Carlo Dropout (MCD; \cite{gal2016dropout}) occupies a prominent place. This method provides an estimate of predictive uncertainty by characterizing the distribution of variations in the model outputs, thereby enabling the derivation of uncertainty bounds for each survival function at arbitrary times of interest, as exemplified in Figure~\ref{fig:MC vehicle_example}. For completeness, we provide beneath a succinct technical overview of the MCD procedure adopted in the present study.

\begin{figure}[H]
    \centering
    \includegraphics[width=0.9\linewidth]{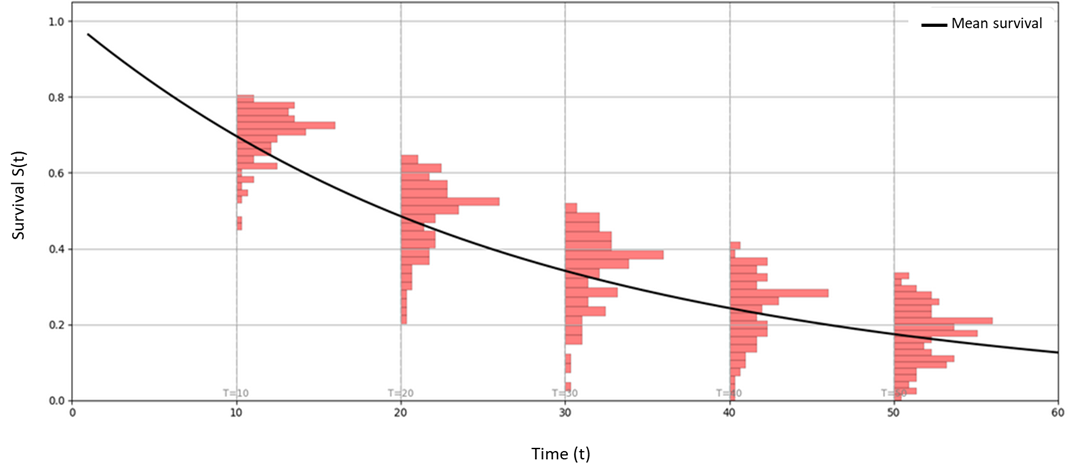}
    \caption{Conditional Weibull survival for several times of interest,  with uncertainty distribution produced by  MCD. This illustration is took from further experiments conducted on a real dataset.}
    \label{fig:MC vehicle_example}
\end{figure}

The key idea of the MCD is to interpret dropout at test time 
as a Bayesian approximation of a deep Gaussian process. Instead of a single deterministic 
prediction, MCD yields a distribution of outputs by repeatedly applying dropout masks during 
inference.

\paragraph{Principle.} 
Let $f_\theta(\mathbf X)$ denote the neural network mapping from covariates $\mathbf X$ to the 
Weibull parameters $(\eta,\beta)$. With dropout rate $p \in (0,1)$, we define the $m$-th stochastic 
forward pass as
\[
(\eta^{(m)}, \beta^{(m)}) \;=\; f_{\theta,\,\text{drop}}^{(m)}(\mathbf X), 
\qquad m = 1,\dots,M,
\]
where $f_{\theta,\,\text{drop}}^{(m)}$ indicates that dropout is applied to the hidden units 
with independent Bernoulli masks. The predictive survival distribution is then approximated by
\[
\widehat S(t \mid \mathbf X) \;=\; \frac{1}{M}\sum_{m=1}^M 
\exp\!\left[-\left(\frac{t}{\eta^{(m)}}\right)^{\beta^{(m)}}\right].
\]
This Monte Carlo estimate provides both a mean survival prediction and an empirical variance, 
which can be interpreted as predictive uncertainty.

\begin{algorithm}[H]
\caption{Monte Carlo Dropout for WTNN survival}
\begin{algorithmic}[1]
\Require Trained WTNN $f_\theta$, covariates $\mathbf X$, number of samples $M$, times of interest $\{t_\ell\}$
\For{$m = 1,\ldots,M$}
  \State Apply dropout mask to hidden units
  \State $(\eta^{(m)},\beta^{(m)}) \gets f_\theta(\mathbf X)$
  \For{each $t_\ell$}
    \State $S^{(m)}(t_\ell) \gets \exp\!\bigl(- (t_\ell / \eta^{(m)})^{\beta^{(m)}}\bigr)$
  \EndFor
\EndFor
\For{each $t_\ell$}
  \State $\widehat S(t_\ell) \gets \dfrac{1}{M}\sum_{m=1}^M S^{(m)}(t_\ell)$
  \State $\widehat{\mathrm{Var}}[S(t_\ell)] \gets \dfrac{1}{M}\sum_{m=1}^M \bigl(S^{(m)}(t_\ell) - \widehat{ S}(t_\ell)\bigr)^2$
\EndFor
\end{algorithmic}
\end{algorithm}

In practice, $M$ is set between 100 and 200 to balance computational cost and estimation 
accuracy. Predictive intervals can then be constructed at each time $t$ using the empirical 
quantiles of $\{S^{(m)}(t)\}_{m=1}^M$.

\subsection{Inverse Probability of Censoring Weighting (IPCW)}\label{appendix:IPCW}

Let $T$ denote the event time, $Y$ the censoring time, and $(Z,\delta)$ the observed data with
$Z=\min(T,Y)$ and $\delta=\mathbf 1\{T\le Y\}$. Let $\mathbf X$ be the covariate vector.
IPCW \cite{IPCWregression} addresses bias from right-censoring by reweighting each observed event by the inverse of
the conditional probability of remaining uncensored up to $Z$.
Defining $S(t\mid \mathbf X)=\mathbb P(Y\ge t\mid \mathbf X)$, for any measurable functional $h$,
\[
\mathbb E\!\left[h(T)\mid \mathbf X\right]
=\mathbb E\!\left[\frac{\delta\,h(Z)}{S(Z\mid \mathbf X)}\;\middle|\; \mathbf X\right].
\]
In practice, $S(\cdot\mid \mathbf X)$ is estimated from the censoring distribution
(e.g., Kaplan--Meier when censoring is independent of $\mathbf X$, or a semiparametric model if it depends on $\mathbf X$).
For learning with a loss $\ell$ defined at event times, IPCW yields the empirical risk
\[
\widehat{\mathcal R}
=\frac{1}{n}\sum_{i=1}^n
\underbrace{\frac{\delta_i}{\widehat S(Z_i\mid \mathbf X_i)}}_{w_i}
\,\ell\!\big(Z_i,\mathbf X_i\big),
\]
where $w_i$ are the inverse-censoring weights. This permits standard estimators or neural losses
to be used consistently in the presence of right-censoring. See also \cite{Cwiling2024} for more details.


\subsection{Appendix: results of training experiments}

\begin{table}[H]
    \centering
    \caption{Training results from simulated experiments for the BN-free WTNN model, trained with the Weibull loss only (averaged over 30 repetitions). }
    \label{tab:simulated_results_Weibull_a}
    \begin{tabular}{l|llllll}
        \hline
        Case & Weight simulation & Bias simulation & IBS & MSE & CI & AUC\\
        \hline
        1 & ${\cal{N}}(0.001, 0.001)$ & ${\cal{N}}(0.01, 0.1) $  & 0.056 & 867 & 0.447 & 0.5\\
        2 & ${\cal{N}}(0.001, 0.001)$  & ${\cal{N}}(1, 0.001)$  &  0.055 & 866 & 0.46 & 0.499\\
        3 & $N(0.01, 0.01)$ & ${\cal{N}}(0.01, 1)$ & 0.053 & 865 & 0.431 & 0.498 \\
        4 & ${\cal{N}}(0.01, 0.01)$  & ${\cal{N}}(10, 10)$ & 0.048 & 811 & 0.4218 & 0.381\\
        5 & ${\cal{N}}(0.01, 0.1)$ & $N(1, 10)$ & 0.051 & 844 & 0.461 & 0.472 \\
        6 & ${\cal{N}}(0.1, 0.01)$ & ${\cal{N}}(0.001, 10)$ & 0.375 &  11,110 & 0.426 & 0.511\\
        7 & $N(0.001, 1)$ & ${\cal{N}}(0.01, 0.01)$ & 0.846 & 137,170 & 0.460 & 0.464 \\
        8 & ${\cal{N}}(1, 0.001)$ & ${\cal{N}}(0.001, 0.01)$ & 0.944 & $ 1.45.10^{8}$ & 0.426 & 0.5 \\
        9 & ${\cal{N}}(10, 0.001)$ & ${\cal{N}}(1, 0.001)$ & 0.944 &  $ 8.0.10^{11}$ & 0.427 & 0.5\\
        \hline
    \end{tabular}
\end{table}

\begin{table}[H]
    \centering
    \caption{Training results from simulated experiments for the WTNN model including BN, trained with the Weibull loss only (averaged over 30 repetitions). The hyphen (–) indicates that the performance indicators value could not be computed. In fact, during the training phase, the loss calculation resulted in infinite values, which prevented the model from training effectively. }
    \label{tab:simulated_results_Weibull-BN}
    \begin{tabular}{l|llllll}
        \hline
        Case & Weight simulation & Bias simulation & IBS & MSE & CI & AUC\\
        \hline
        1 & ${\cal{N}}(0.001, 0.001)$ & ${\cal{N}}(0.01, 0.1) $  & - & - & - & -\\
        2 & ${\cal{N}}(0.001, 0.001)$  & ${\cal{N}}(1, 0.001)$  & - & - & - & -\\
        3 & $N(0.01, 0.01)$ & ${\cal{N}}(0.01, 1)$ & 0.056 & 867 & 0.456 & 0.499\\
        4 & ${\cal{N}}(0.01, 0.01)$  & ${\cal{N}}(10, 10)$ & 0.054 & 866 & 0.487 & 0.499\\
        5 & ${\cal{N}}(0.01, 0.1)$ & $N(1, 10)$ & - & - & - & - \\
        6 & ${\cal{N}}(0.1, 0.01)$ & ${\cal{N}}(0.001, 10)$ & 0.056 &  867 & 0.478 & 0.499\\
        7 & $N(0.001, 1)$ & ${\cal{N}}(0.01, 0.01)$ & - & - & - & - \\
        8 & ${\cal{N}}(1, 0.001)$ & ${\cal{N}}(0.001, 0.01)$ & 0.056 & 867 & 0.518 & 0.499 \\
        9 & ${\cal{N}}(10, 0.001)$ & ${\cal{N}}(1, 0.001)$ & 0.056 &  867 & 0.516 & 0.499 \\
        \hline
    \end{tabular}
\end{table}

\begin{table}[H]
    \centering
    \caption{Training results from simulated experiments for the BN-free WTNN model, trained with the whole loss function (averaged over 30 repetitions). }
    \label{tab:simulated_results_Weibull_MSE}
    \begin{tabular}{l|llllll}
        \hline
        Case & Weight simulation & Bias simulation & IBS & MSE & CI & AUC\\
        \hline
        1 & ${\cal{N}}(0.001, 0.001)$ & ${\cal{N}}(0.01, 0.1) $  & 0.056 & 867 & 0.514 & 0.5\\
        2 & ${\cal{N}}(0.001, 0.001)$  & ${\cal{N}}(1, 0.001)$  &  0.056 & 867 & 0.5 & 0.5\\
        3 & $N(0.01, 0.01)$ & ${\cal{N}}(0.01, 1)$ & 0.056 & 866 & 0.464 & 0.498 \\
        4 & ${\cal{N}}(0.01, 0.01)$  & ${\cal{N}}(10, 10)$ & 0.048 & 805 & 0.447 & 0.419\\
        5 & ${\cal{N}}(0.01, 0.1)$ & $N(1, 10)$ & 0.045 & 808 & 0.452 & 0.411 \\
        6 & ${\cal{N}}(0.1, 0.01)$ & ${\cal{N}}(0.001, 10)$ & 0.051 &  834 & 0.439 & 0.44\\
        7 & $N(0.001, 1)$ & ${\cal{N}}(0.01, 0.01)$ & 0.874 & 275,699 & 0.470 & 0.524 \\
        8 & ${\cal{N}}(1, 0.001)$ & ${\cal{N}}(0.001, 0.01)$ & 0.944 & $ 4.04.10^{7}$ & 0.424 & 0.499 \\
        9 & ${\cal{N}}(10, 0.001)$ & ${\cal{N}}(1, 0.001)$ & 0.944 &  $ 3.09.10^{11}$ & 0.427 & 0.5\\
        \hline
    \end{tabular}
\end{table}

\begin{table}[H]
    \centering
    \caption{Training results from simulated experiments for the BN WTNN model, trained with the whole loss function (averaged over 30 repetitions). }
    \label{tab:simulated_results_Weibull_MSE_BN}
    \begin{tabular}{l|llllll}
        \hline
        Case & Weight simulation & Bias simulation & IBS & MSE & CI & AUC\\
        \hline
        1 & ${\cal{N}}(0.001, 0.001)$ & ${\cal{N}}(0.01, 0.1) $  & - & - & - & - \\
        2 & ${\cal{N}}(0.001, 0.001)$  & ${\cal{N}}(1, 0.001)$   & 0.056 & 867 & 0.517 & 0.499\\
        3 & $N(0.01, 0.01)$ & ${\cal{N}}(0.01, 1)$ & 0.055 & 866 & 0.504 & 0.499\\
        4 & ${\cal{N}}(0.01, 0.01)$  & ${\cal{N}}(10, 10)$  & - & - & - & -\\
        5 & ${\cal{N}}(0.01, 0.1)$ & $N(1, 10)$ & 0.056 & 867 & 0.502 & 0.499 \\
        6 & ${\cal{N}}(0.1, 0.01)$ & ${\cal{N}}(0.001, 10)$ & 0.056 &  867 & 0.452 & 0.499\\
        7 & $N(0.001, 1)$ & ${\cal{N}}(0.01, 0.01)$ & - & - & - & - \\
        8 & ${\cal{N}}(1, 0.001)$ & ${\cal{N}}(0.001, 0.01)$ & 0.056 & 867 & 0.446 & 0.499 \\
        9 & ${\cal{N}}(10, 0.001)$ & ${\cal{N}}(1, 0.001)$ & 0.056 &  867 & 0.449 & 0.499 \\
        \hline
    \end{tabular}
\end{table}

\end{document}